\documentclass{article}

\usepackage{arxiv}

\usepackage[utf8]{inputenc} 
\usepackage[T1]{fontenc}    
\usepackage{hyperref}       
\usepackage{url}            
\usepackage{booktabs}       
\usepackage{amsfonts}       
\usepackage{nicefrac}       
\usepackage{microtype}      
\usepackage{lipsum}
\usepackage{graphicx}
\graphicspath{ {./images/} }

\usepackage{amsmath}
\usepackage{amssymb}
\usepackage{mathtools}
\usepackage{amsthm}

\usepackage[capitalize,noabbrev]{cleveref}

\usepackage{color}
\usepackage{multirow}
\usepackage[labelsep=colon]{caption} %
\usepackage{natbib}
\usepackage{algorithm}
\usepackage{algorithmic}
\usepackage{adjustbox}
\usepackage{tabularx}
\usepackage{svg}
\usepackage{wrapfig}
\usepackage{float}      
\usepackage{datetime}
\usepackage{url}
\usepackage{amsfonts}
\usepackage{nicefrac}
\usepackage{microtype}
\usepackage{bm}
\usepackage{mathrsfs}
\usepackage{enumitem}
\usepackage{hhline}
\usepackage{arydshln}
\usepackage{accents}
\usepackage{trimclip}
\usepackage{dsfont}
\usepackage{bbding}
\usepackage{setspace}
\usepackage{array, pifont}
\usepackage{makecell}
\usepackage{colortbl}       
\usepackage[table]{xcolor}  
\makeatletter
\@ifpackageloaded{xcolor}{\def\XC@usecolor@}{\XC@usecolor{}}\relax
\makeatother
\usepackage{afterpage}
\usepackage{pbox}
\usepackage{silence}

\theoremstyle{plain}
\newtheorem{theorem}{Theorem}[section]

\theoremstyle{definition}

\newtheorem{assumption}[theorem]{Assumption}
\theoremstyle{remark}
\newtheorem{remark}[theorem]{Remark}

\usepackage{amsmath,amsfonts,bm}

\def\eqref#1{(\ref{#1})}
\def\Eqref#1{Equation~\ref{#1}}

\def\1{\bm{1}}

\def\ve{{\bm{e}}}

\def\vu{{\bm{u}}}

\def\vw{{\bm{w}}}
\def\vx{{\bm{x}}}
\def\vy{{\bm{y}}}
\def\vz{{\bm{z}}}

\def\beps{{\boldsymbol{\epsilon}}}

\DeclareMathAlphabet{\mathsfit}{\encodingdefault}{\sfdefault}{m}{sl}
\SetMathAlphabet{\mathsfit}{bold}{\encodingdefault}{\sfdefault}{bx}{n}

\def\gN{{\mathcal{N}}}

\newcommand{\etens}[1]{\mathsfit{#1}}

\def\etC{{\etens{C}}}

\def\etE{{\etens{E}}}

\def\etQ{{\etens{Q}}}

\def\etW{{\etens{W}}}

\def\vzhat{{\widehat{\vz}}}
\def\vxbar{{\bar{\vx}}}
\def\vzbar{{\bar{\vz}}}

\def\Tsf{{\mathsf{T}}}
\def\veta{{\bm{\eta}}}

\newcommand{\E}{\mathbb{E}}

\newcommand{\R}{\mathbb{R}}

\renewcommand{\Pr}{\mathbb{P}}

\usepackage[textsize=tiny]{todonotes}
\definecolor{darkred}{RGB}{180,0,0}
\definecolor{cvprblue}{rgb}{0.21,0.49,0.74}
\definecolor{lightgreen}{rgb}{.9,1,.9}
\definecolor{lightblue}{rgb}{.9,.9,1.}

\newcommand{\cosim}{\operatorname{cosim}}

\title{Unpaired Image-to-Image Translation via a Self-Supervised Semantic Bridge}

\author{
Jiaming Liu$^{1}$~\thanks{Corresponding authors. Questions
to: jiamliu@stanford.edu}~~~~
Felix Petersen$^{1}$~~~~
Yunhe Gao$^{1}$~~~~
Yabin Zhang$^{1}$~~~~
\\
\bfseries
Hyojin Kim$^{2}$~~~~
Akshay S. Chaudhari$^{1}$~\footnotemark[1]~~~
Yu Sun$^{3}$~\footnotemark[1]~~~~
Stefano Ermon$^{1}$~~~~
Sergios Gatidis$^{1}$
\\[0.25em]
$^{1}$Stanford University \quad
$^{2}$LLNL \quad
$^{3}$Johns Hopkins University
}

\renewcommand{\today}{}
\begin{document}
\maketitle

\begin{abstract}
Adversarial diffusion and diffusion-inversion methods have advanced unpaired image-to-image translation, but each faces key limitations. Adversarial approaches require target-domain adversarial loss during training, which can limit generalization to unseen data, while diffusion-inversion methods often produce low-fidelity translations due to imperfect inversion into noise-latent representations. In this work, we propose the Self-Supervised Semantic Bridge (SSB), a versatile framework that integrates external semantic priors into diffusion bridge models to enable spatially faithful translation without cross-domain supervision. Our key idea is to leverage self-supervised visual encoders to learn representations that are invariant to appearance changes but capture geometric structure, forming a shared latent space that conditions the diffusion bridges. Extensive experiments show that SSB outperforms strong prior methods for challenging medical image synthesis in both in-domain and out-of-domain settings, and extends easily to high-quality text-guided editing. Code and models are publicly available in \href{https://github.com/StanfordMIMI/SemanticBridge}{\color{blue}{here}}.
\end{abstract}


\section{Introduction}
Unpaired image-to-image (I2I) translation poses a critical challenge in unsupervised representation learning: how to disentangle and transfer semantic content across distinct domains without explicit correspondence~\citep{park2020contrastive}. This capability is fundamental to computer vision and medical imaging, enabling tasks such as medical image synthesis~\cite{ozbey2023unsupervised, zhan2024medm2g, wang2025vastsd, lyu2022conversion, xu2024medsyn, jiang2023cola, wang2025self, croitoru2023diffusion} and general image editing~\cite{gatys2016image, ling2021editgan, tumanyan2022splicing, zhang2023adding, alami2018unsupervised, zhu2017unpaired, liu2017unsupervised}. Despite significant progress, existing approaches still face fundamental challenges in balancing distributional robustness and structural preservation under limited paired supervision. 

Early GAN-based methods~\citep{goodfellow2014generative, zhu2017unpaired, liu2017unsupervised, huang2018munit, choi2018stargan} optimize adversarial objectives—often combined with cycle-consistency constraints—to learn cross-domain mappings and achieve promising translation results. Building on this, recent approaches incorporate diffusion models~\citep{song2020score, ho2020denoising} for conditional image translation~\citep{kim2024unpaired, kim2024pagoda, zhang2025disdiff, xu2023cyclenet, ozbey2023unsupervised}. These methods improve translation realism through explicit cross-domain regularization on unpaired data during diffusion model training. However, both families depend on explicit coupling between source and target domains, limiting their scalability and generalization across diverse distributions. For example, in MRI$\rightarrow$CT synthesis, variations in MRI contrast often lie beyond the training distribution, leading to reduced performance to out-of-domain (OOD) data. This motivates decoupling translation from cross-domain training by introducing a shared semantic interface that connects domains without explicit alignment objectives

Alternatively, inversion-based approaches~\citep{meng2022sdedit,su2022dual,huberman2024edit,rout2025semantic,wu2023latent} translate images by inverting them into the latent noise space of a pretrained diffusion model and re-synthesizing under target-domain conditioning. In practice, inversion is approximate and errors propagate through sampling, often resulting in structural drift from the source. Recent methods mitigate this by injecting intermediate features—e.g., reusing attention maps recorded during inversion~\citep{parmar2023zero,tumanyan2023plug,chung2024style,Avrahami_2025_CVPR}—but these interventions are usually tied to specific architectures or sampling procedures, which limits transferability across methods.

Bridge models~\citep{zhou2024denoising,zhang2024exploring,pmlr-v202-liu23ai,xiao2025deterministic} and stochastic interpolants~\citep{albergo2023stochastic} learn stochastic or deterministic paths between arbitrary distributions and achieve high-fidelity translation when paired supervision is available~\citep{chadebec2025lbm,pmlr-v202-liu23ai, arslan2025self}. A related line of work targets unpaired translation via distribution alignment, including optimal transport~\cite{korotin2023neural,mokrov2024energyguided,choi2025improving}, Schrödinger bridges~\cite{shi2023diffusion,gushchin2024adversarial, ksenofontov2025categorical}, and bridge distillation~\cite{lee2025single}. While promising, scaling these unpaired approaches to complex high-dimensional regimes, such as self-supervised medical synthesis and text-guided editing, remains underexplored.

In this work, we propose a new, fully self-supervised solution for unpaired image-to-image translation, called the \emph{\textbf{S}elf-supervised \textbf{S}emantic \textbf{B}ridge (SSB)}. SSB addresses practical settings where paired data are scarce and test-time distribution shifts are pervasive, aiming to strictly preserve source fidelity while enabling high-quality cross-domain translation. Our key insight is to perform unpaired translation through a \emph{shared semantic manifold} derived from self-supervised visual encoders' patch embeddings (e.g., DINOs~\cite{caron2021emerging, oquab2024dinov, simeoni2025dinov3}). These embeddings provide a geometry-consistent semantic interface between domains, enabling \emph{independent} self-supervised training of domain-specific generative models without cross-domain alignment or adversarial loss. As a result, extending SSB to additional domains requires training only one new single-domain model, yielding linear scaling in the number of domains rather than the pairwise cost of domain-coupled approaches. We summarize the main contributions as follows:
\vspace{-0.5em}
\begin{itemize}
    \item We introduce \emph{SSB}, a simple but effective framework for unpaired image-to-image translation that connects domains through a self-supervised shared semantic latent space with \emph{independent} per-domain training, supported by theoretical justification.

    \item We develop a geometry-aware MRI--CT representation via DINOv2 pre-training, enabling SSB to achieve strong unpaired MRI$\rightarrow$CT translation in both in-domain and out-of-domain settings, with performance comparable to supervised approaches.

    \item We extend SSB to natural-image translation and text-guided editing, achieving competitive performance on both scene transfer and object-level editing.
\vspace{-1.em}    
\end{itemize}

\begin{figure}[t!] 
    \centering
    \includegraphics[width=0.9\linewidth]{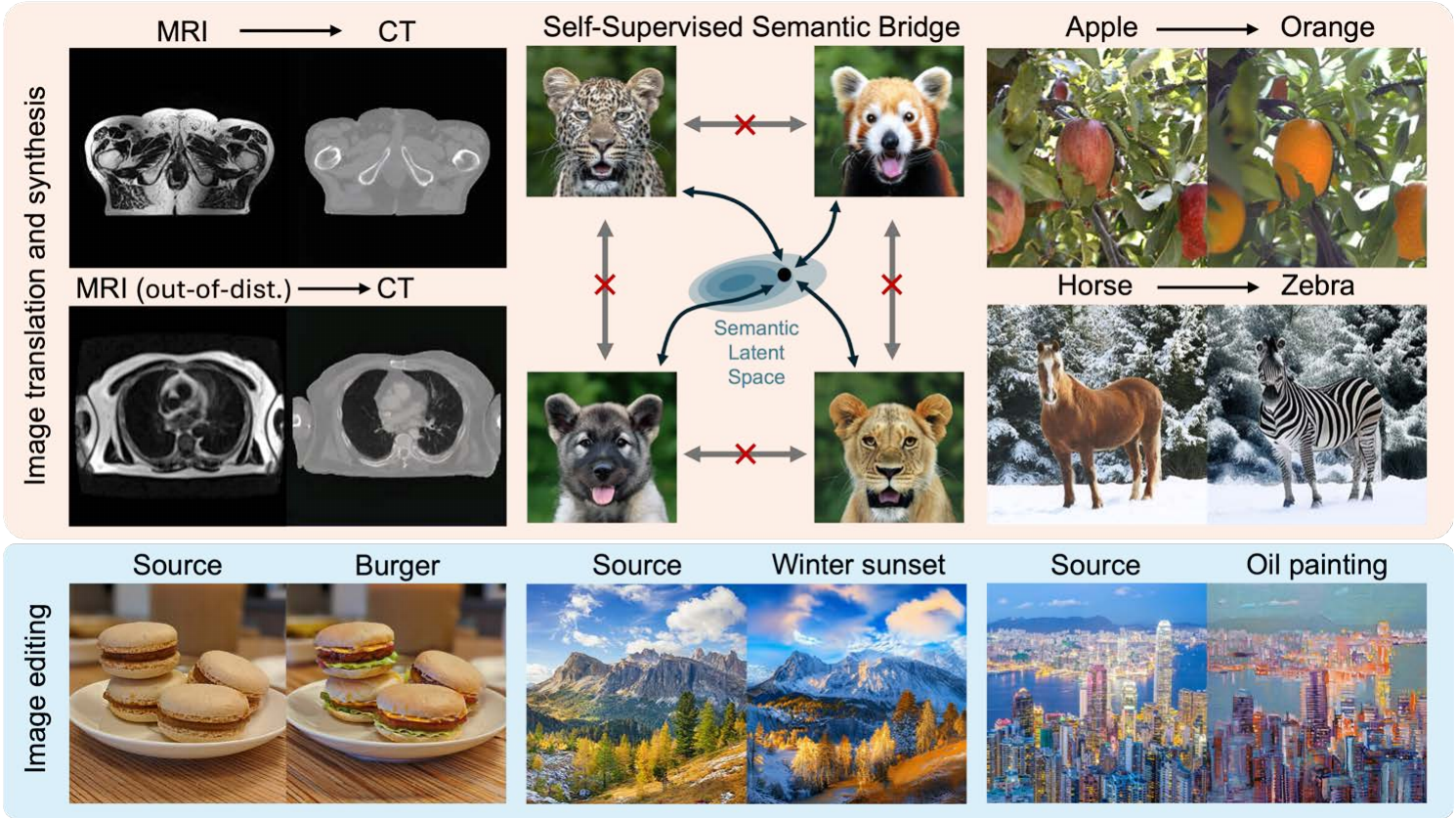}
    \vspace{-0.0em}
\caption{Overview of our~\textbf{Self-Supervised Semantic-Bridge (SSB)} framework for unpaired image translation and editing. SSB trains without paired data or adversarial objectives, relying on a shared latent-space assumption to connect domains via a common representation; \textcolor{darkred}{\boldmath$\times$} denotes no cross-domain supervision. The top rows showcase image-to-image translation results using text-free guidance, while the bottom rows demonstrate text-guided editing based on Stable Diffusion 3-Medium~\cite{esser2024scaling}. This single-domain training enables unpaired translation across medical and natural images with improved structural consistency and OOD robustness (e.g., unseen MRI contrasts at test-time).
}
\vspace{-1em}    
    \label{fig:teaser}
\end{figure}

\section{Related Work}
\label{Sec:related works}
\noindent
\noindent
\textbf{Unpaired I2I Methods.} Unpaired image translation methods commonly rely on cycle-consistency, adversarial learning, or their combination~\cite{goodfellow2014generative,zhu2017unpaired,liu2017unsupervised,huang2018munit,choi2018stargan,park2020contrastive,zheng2021spatially,xu2023cyclenet,ozbey2023unsupervised,kim2024unpaired}.
Unlike these approaches that align domains through direct coupling or discriminator feedback at training, USBM constructs a shared semantic latent space via self-supervised representation learning and learns domain-specific bridges without explicit cross-domain supervision.
Alternatively, inversion-based diffusion and flow models~\cite{meng2022sdedit,song2021denoising,su2022dual,rout2025semantic} leverage powerful pre-trained generative models by mapping images to noise latents and re-synthesizing them under new conditions, later extended to text-guided and conditional tasks via feature injection (~\emph{e.g.}, cross-attention)~\citep{parmar2023zero,tumanyan2023plug,Avrahami_2025_CVPR}.
Recent inversion-free variants~\citep{kulikov2024flowedit,kim2025flowalign} further improve translation consistency using advanced flow backbones such as Flux~\cite{labs2025flux1kontextflowmatching} and SD3~\cite{esser2024scaling}.
Although conceptually related in their goal of injecting source structures during sampling, USBM departs from these methods by aligning self-supervised and diffusion representations to achieve structure-preserving translation within a unified bridge formulation.

\noindent
\textbf{Diffusion Models with External Representations.}Recent studies have explored leveraging external representations to improve the controllability, efficiency, and performance of diffusion models across image-related tasks. 
One line of work uses \emph{structured conditions}---such as edges, depth---as explicit geometric guidance for controllable generation and editing~\cite{zhang2023adding, chen2024contourdiff}. 
A complementary line of work uses \emph{learned embeddings} from pretrained vision encoders to provide semantic priors, benefiting image synthesis~\cite{pernias2024wrstchen, li2024return, yu2024representation}, image matching and I2I editing~\cite{zhang2023tale, xue2025matcha}, and depth estimation~\cite{bai2025fiffdepth}. 
REPA~\cite{yu2024representation} aligns diffusion feature similarities between diffusion transformers and DINOv2 for efficient training, while SD-DINO~\cite{zhang2023tale} fuses DINOv2 features into diffusion models at test time for appearance-level I2I translation. In contrast to these methods, SSB introduces a unified bridge formulation that leverages self-supervised representations as a geometry-consistent semantic interface to learn domain transitions from unpaired data, enabling faithful structure preservation during translation and editing without cross-domain alignment or adversarial objectives.

\section{Preliminaries}
\label{Sec:Preliminaries}

\noindent
\textbf{Diffusion Bridge Models.} Bridge models~\citep{zhou2024denoising, zheng2025diffusion, shi2023diffusion} and Stochastic interpolants (SI)~\citep{albergo2023building,albergo2023stochastic} unify denoising diffusion models and rectified flows by extending beyond Gaussian priors. They interpolate between arbitrary data $\vx_0 \sim q_{\text{data}}$ and prior samples $\vx_T \sim p_{\text{prior}}$. In the linear–Gaussian case, diffusion bridges reduce to SIs, sharing the same conditional marginals $p(\vx_t \mid \vx_0,\vx_T)$ and reverse dynamics~\citep{albergo2023stochastic, zhang2024exploring},
\begin{equation}
\vx_t = I(t,\vx_0,\vx_T) + \gamma_t \beps,
\quad \beps \sim \gN(\mathbf{0},\mathbf{I}), ~~ t \in [0,T].
\label{eq:stochastic-interpolants}
\end{equation}
Without loss of generality, we set $T = 1$ by time reparameterization, ensuring $\gamma_0 = \gamma_T = 0$ and boundary consistency \(I(0) = \vx_0, I(T) = \vx_T\). A typical choice is the linear interpolant $I(t)=\alpha_t\vx_0+\beta_t\vx_T$, which yields the stochastic velocity
$v(t,\vx_t) = \partial_t I(t,\vx_0,\vx_T) + \dot\gamma_t\beps$.
The probability–flow velocity removes noise via conditioning, $v^\star(t,\vx_t)=\E[v\mid \vx_t]$, which is intractable and thus approximated by a neural net $v_\theta(t,\vx_t)$. Sampling then follows the probability flow ODE (PF-ODE),
\begin{equation}
d\vx_t = v_\theta(t,\vx_t)dt.
\label{eq:PF-ODE-reverse}
\end{equation}
When $\gamma_t \equiv 0$ and $\vx_T \sim \gN(\mathbf{0},\mathbf{I})$, this reduces to rectified flows~\citep{liu2023flow, lipman2023flow}. Ideally, integrating the forward field $v_\theta(t)$ transports an image $\vx_0$ to its terminal latent representation $\vx_T$, while reverse integration (sampling) reconstructs $\vx_0$, forming a deterministic encoder–decoder pair~\cite{song2021denoising}. 

\textbf{Self-Supervised Visual Encoders.} 
The DINO~\cite{caron2021emerging, oquab2024dinov} family employs self-distillation to learn spatially-aware patch representations. Given an input $\vx$, the encoder $T_\phi$ produces a \texttt{[CLS]} token $T^{\texttt{[CLS]}}_\phi(\vx)$ and patch tokens $T^{\text{patch}}_\phi(\vx)$. Predictions are then obtained from the \texttt{[CLS]} token via a softmax over $K$ learned prototypes. During training, DINOs sample two independent augmentations $A_1, A_2$ of the same image. A teacher network $T_{\phi'}$ (an EMA of $T_\phi$) provides the target distribution, and the student $T_\phi$ is optimized to match it through the following cross-entropy loss function
\begin{equation}
\mathcal L({\phi;\phi'})=
\mathbb E
\Big[ -p_{\phi'}(A_2(\vx))^\top \log p_\phi(A_1(\vx)) \Big].
\label{eq:ce-loss}
\end{equation}
This objective encourages invariance to local perturbations, yielding robust semantic representations that we leverage to construct our shared latent space in the next section.

\section{Proposed Method}
\label{Sec:Method}
\subsection{Shared Latent Space Assumption}
\label{SubSec:shared-latent}

Inspired by I2I translation~\citep{liu2017unsupervised, huang2018munit, su2022dual}, 
we assume that multi-domain observations $(\vx^{(1)}, \ldots, \vx^{(M)}) \sim q_{\text{data}}$ share a common latent representation. Although the true joint distribution $q_{\text{data}}$ is unknown, we posit a shared latent variable $\vy \sim p_{\text{prior}}$ 
that captures semantic content aligned across domains, yielding,
\begin{equation}
p(\vz^{(1)}, \ldots, \vz^{(M)}, \vy) 
= p(\vy)\,\prod_{i=1}^M p^{(i)}(\vz^{(i)} \mid \vy),
\label{eq:joint-factorization}
\end{equation}
where $\vz = {E_\varphi}(\vx)$ are VAE~\cite{rombach2022high} latents and 
$p^{(i)}(\cdot \mid \vy)$ denotes the conditional distribution for domain $i$.
Conditioned on $\vy$, the domains are independent. To formalize ``shared semantics,'' it is convenient to introduce an \emph{oracle} encoder $E^\ast$ such that
for any semantically corresponding images $(\vx^{(i)}, \vx^{(j)})$, $E^\ast(\vx^{(i)}) \approx E^\ast(\vx^{(j)}) \approx \vy$. In practice, we approximate $E^\ast$ with a fixed pretrained network $E_\phi$. Importantly, our theoretical error analysis and empirical results on medical and natural images below demonstrate that our assumption in ~\eqref{eq:joint-factorization} remains effective despite the domain gap in the pre-trained feature space.

\smallskip\noindent
\textbf{Cross-domain Translation using Bridge Models.} 
Translation from domain $j$ to $i$ is then expressed as an integral over the shared latent. Since this posterior is intractable, we approximate it by a Dirac mass at the encoder output $E^\ast$,

\begin{equation}
\begin{aligned}
p(\vz^{(i)} \mid \vz^{(j)}) \approx p_\theta^{(i)}\!\big(\vz^{(i)} \mid \vy = E^\ast(\vx^{(j)})\big).
\end{aligned}
\vspace{-0em}
\label{eq:cross-modal}
\end{equation}
where $p_\theta^{(i)}$ is a diffusion bridge that maps the shared latent $\vy$ to the domain-$i$ VAE latents. Its reverse process thus enables a clear constructive inference: encoding a source image \(\vx^{(j)}\) into \(\vy\), sampling from \(\bar\vz^{(i)} \sim p_\theta^{(i)}(\cdot \mid \vy)\), and decoding to the target image \(\bar\vx^{(i)} = D_\varphi(\bar\vz^{(i)})\).

\subsection{Shared Latent Space Encoders}
\label{SubSec:shared-latent-encoder}

\begin{figure}[t!] 
    \centering
    \includegraphics[width=.6\linewidth]{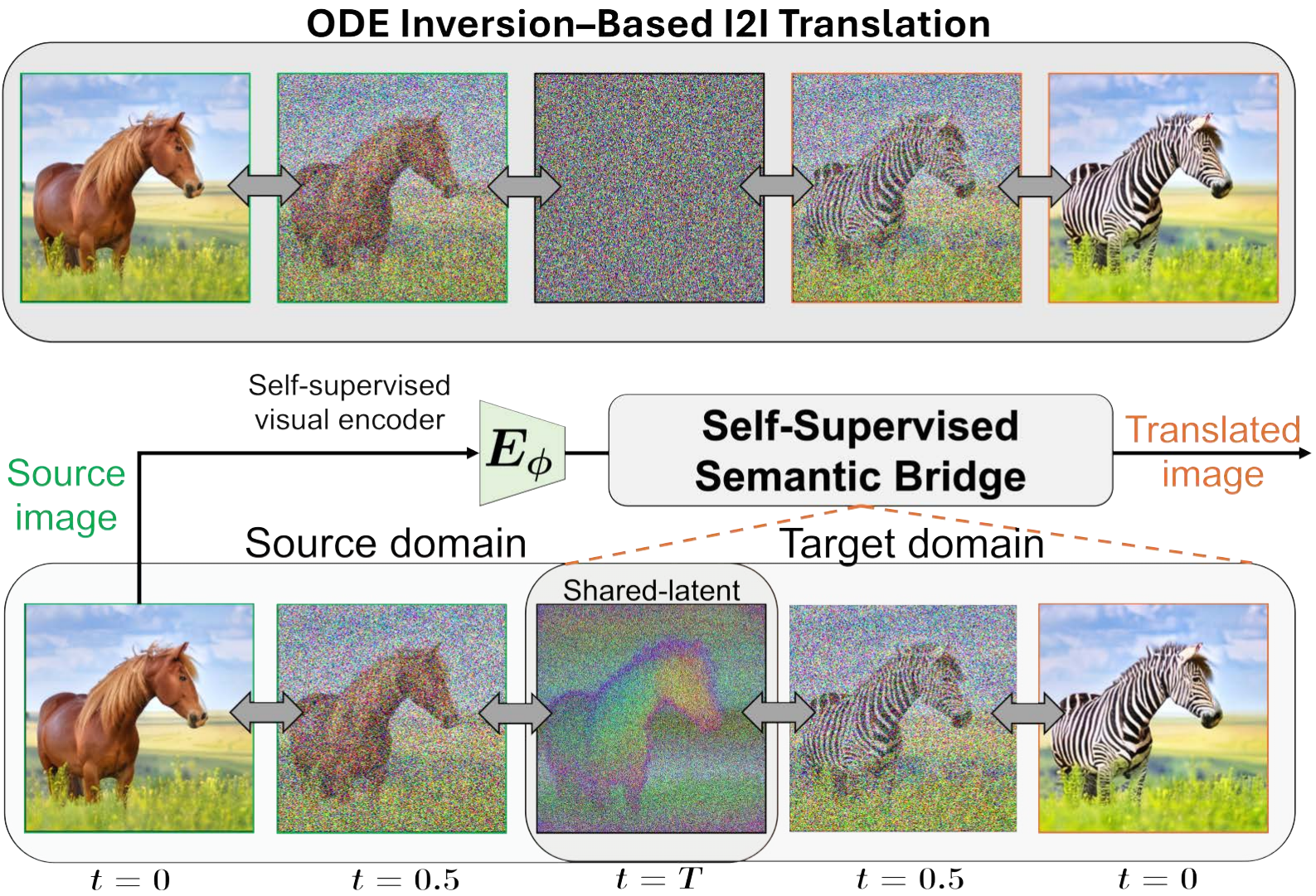}
    \vspace{-0.0em}
\caption{
Unlike inversion-based methods that invert toward an unstructured Gaussian noise, \textbf{SSB} defines a unified semantic latent endpoint $\vy = E_{\phi}(\vx)$ using a self-supervised visual encoder and trains domain-specific bridges independently to connect each domain to this shared endpoint, enabling \emph{reliable} translation by composing source-to-$T$ inversion with target-bridge generation.
}
\vspace{-1em}    
    \label{fig:pipline}
\end{figure}

\smallskip\noindent
The key ingredient of our method is a practical proxy for the semantic encoder $E^\ast$, built from DINO-ViT features. Specifically, we factor augmentations as $A = G \circ C$, where $C \in \mathcal C$ applies appearance transforms (e.g., color, contrast, etc) and $G \in \mathcal G$ applies geometric transforms (e.g., crop, scale, etc). Near convergence, minimizing the student-teacher consistency loss in~\eqref{eq:ce-loss} biases the learned predictor to be approximately insensitive to appearance change,
\[
p_{\phi}(G \circ C(\vx)) \approx p_{\phi}(G(\vx)),
\qquad \forall C \in \mathcal C,\; G \in \mathcal G, 
\]
for almost every input image $\vx$. This appearance insensitivity makes DINO features suitable for settings with strong structure correspondence but large appearance gaps (e.g., MRI--CT).
Moreover, since the \texttt{[CLS]} prediction is computed via self-attention over patch tokens, the same training signal propagates to patch-level representations, which retain richer local spatial structure than the global token.
Accordingly, we define the shared latent encoder as
$E_{\phi}(\vx) = P(T^{\text{patch}}_\phi(\vx)),$
where $P$ is a linear PCA projection used for dimensional alignment.
In practice, the number of retained components $B$ (e.g., 8 or 16) is chosen to match the KL-VAE latent dimensionality, so that the resulting components align with VAE channels while emphasizing geometry-consistent factors and suppressing residual appearance noise.

\subsection{Latent Bridge Models as Conditional Decoders}

In principle, a bridge trajectory can be constructed directly from the shared embedding $\vz_T = E_{\phi}(\vx^{(j)})$ to the target latent $\vz_0^{(i)}$. However, this rigid mapping often fails under strong appearance ambiguities. To address this, we introduce a \emph{unified endpoint formulation} that adapts to the domain characteristics: $\vz_T^{(i)} \sim \gN\big(E_{\phi}(\vx^{(i)}), b^2 \mathbf{I}\big)$, where $b$ controls endpoint uncertainty. For geometry-dominated tasks (e.g., MRI$\rightarrow$CT) where the encoder accurately captures the shared structure, we set $b=0$, reducing the model to a simple deterministic map that strictly preserves fidelity. Conversely, for appearance-ambiguous tasks (e.g., natural images), we employ a stochastic endpoint ($b>0$), which enables the PF-ODE to refine the noisy state toward the geometric center while synthesizing valid domain details. Thus, while our framework is general enough to handle diverse modalities, it naturally simplifies to the minimal necessary complexity for each specific task.

For clarity, we omit the domain index in the following sections when it is clear. With pinned endpoints \((\vz_0,\vz_T)\), we have the latent transition kernel with normalized time step $T=1$ for a single domain-$i$ as
\begin{align}
p(\vz_t \mid \vz_0, \vz_T) 
&= \gN\!\left(\vz_t ; I(t, \vz_0, \vz_T),\, \gamma_t^2 \mathbf{I}\right),
\label{eq:latent-SI-forward-kernel}
\end{align}
where $I(t, \vz_0, \vz_T) := \alpha_t \vz_0 + \beta_t \vz_T,\quad t \in [0,1].$ This corresponds to the latent stochastic interpolants
\(\vz_t = \alpha_t \vz_0 + \beta_t \vz_T + \gamma_t \beps,\;\; \beps \sim \mathcal{N}(\mathbf{0},\mathbf{I})\), which induces a PF-ODE $\dot\vz_t = v(t,\vz_t,\vz_T)$, with $v(t,\vz,\vz_T)
    = \dot\alpha_t \E[\vz_0|\vz_t = \vz, \vz_T] + \dot\beta_t\vz_T + \dot\gamma_t\E[\beps|\vz_t = \vz, \vz_T]$. Without loss of generality, we train a network $v_\theta$ to approximate the $v$ as
\begin{equation}
\begin{aligned}
\mathcal L_{\theta}=\E_{t,\vz_0,\vz_T}
\Big[\big\| v_\theta(t,\vz_t,\vz_T) 
- (\partial_tI(t,\vz_0,\vz_T) + \dot\gamma_t \beps) ||^2\Big],
\end{aligned}
\label{eq:SI-conditonal-mean-training-objective}
\end{equation}    
where $t\in[\varepsilon, 1-\varepsilon]$ with a small $\varepsilon$, ensuring numerical stability of $\dot{\gamma}_t$ near the endpoints. Alternatively, one may approximate $v$ by first learning a network to predict the conditional mean $\E[\vz_0|\vz_t,\vz_T]$~\cite{zhou2024denoising, zhang2024exploring}, and subsequently computing the full velocity field in the closed form.

\smallskip\noindent
\textbf{Reverse ODE for Conditional Translation.} The cross domain translation $p(\vz^{(i)}|\vz^{(j)})$ in Eq.~\ref{eq:cross-modal} can then be formulated as a conditional bridge $p_\theta^{(i)}\!\big(\vz_0^{(i)} \mid \vz_T = \vy (\vx_0^{(j)})\big)$, where the shared latent $\vy$ can be obtained deterministically $\vy=E_\phi(\vx_0^{(j)})$ or through PF-ODE inversion $\vy=\vz_T^{\text{inv}}(\vx_0^{(j)})$, where $\vz_T^{\text{inv}}(\vx_0^{(j)})$ is obtained by integrating the forward ODE $d\vz_t^{(j)} = v^{{(j)}}(t,\vz_t^{(j)})dt$ from $t=0$ to $t=T$. As a result, the reverse ODE of the bridge process for target domain-$i$,
\begin{equation}
d\vz_t^{(i)} = v^{(i)}\left(t, \vz_t^{(i)}, \vz_T\right)dt,~~\vz_T = \vy (\vx_0^{(j)}),
\label{eq:rev-ode-generic}
\end{equation}
is solved backward in time, starting from $t = T$.

\smallskip\noindent

\begin{figure*}[t] 
    \centering
    \includegraphics[width=\linewidth]{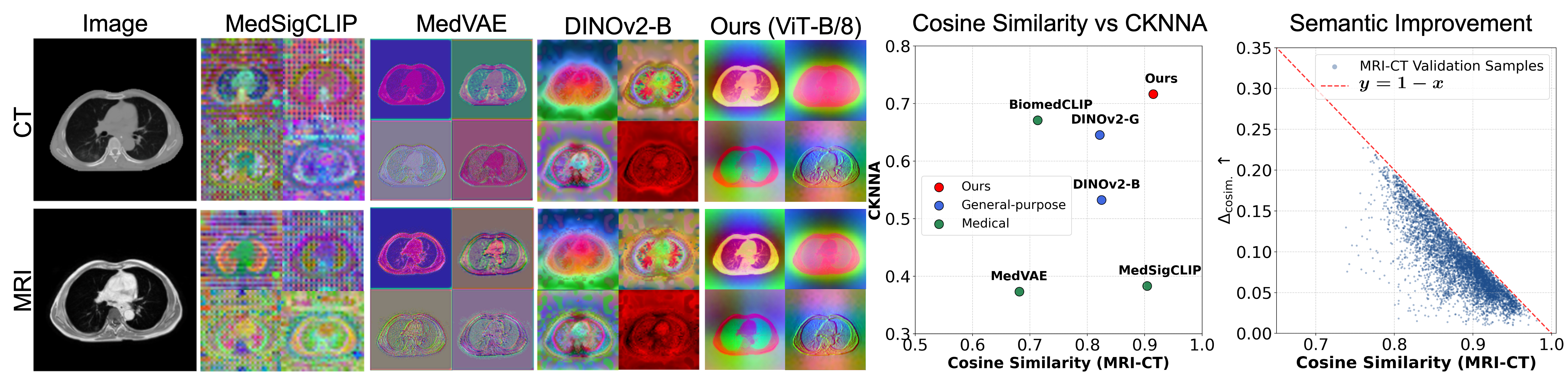}
    \vspace{-1em}
    \caption{\textbf{Cross-modal semantic alignment and improvement in the shared latent space.}
    \textbf{Left:} PCA visualizations of intermediate feature maps for paired CT/MRI inputs. Our ViT-B/8 encoder yields smoother and more anatomically consistent representations across modalities.
    \textbf{Middle:} Quantitative alignment on paired MRI--CT validation samples, comparing global alignment (cosine similarity) and local structural consistency (CKNNA).
\textbf{Right:} Semantic improvement $\Delta_{\text{cosim}}$ measures the net gain in feature alignment with the ground-truth CT ($\text{gtCT}$) achieved by our generated translation ($\text{genCT}$) over the original input ($\text{MRI}$).
Formally, $\Delta_{\text{cosim}} = \cosim(E_\phi(\text{genCT}), E_\phi(\text{gtCT})) - \cosim(E_\phi(\text{MRI}), E_\phi(\text{gtCT}))$.
We plot this gain against the initial input alignment $x$. The dashed line $y = 1 - x$ represents the theoretical ceiling, corresponding to perfect recovery of the ground-truth features (similarity $= 1$). The improvements confirm that our method effectively minimizes semantic divergence, largely mitigating the impact of the initial domain gap.}
    \label{fig:medical_encoders}
\end{figure*}
\subsection{Translation Error Analysis}
In this section, we present a theoretical analysis of translation error arising from inaccuracies in the shared latent approximation introduced in Sec.~\ref{SubSec:shared-latent-encoder} and examine how such errors propagate through the ODE flow during translation. Our goal is to quantify the deviation between the ground-truth target $\vx_0$ and the predicted translation $\bar{\vx}_0$. For clarity, we focus on the ODE formulation of linear bridge models under Euler discretization, which serves as the foundation for many first-order solvers used in diffusion-based sampling. We make basic but necessary assumptions throughout the analysis below. 

\begin{theorem}
\label{prop:translation_error_main}
Given target and source images that share a latent code under a learned vision encoder $E_\phi$.
Assume the learned vector field $v_\theta$ and decoder $D_\varphi$ are Lipschitz-continuous, and both the encoder and decoder introduce bounded reconstruction errors.
If the bridge ODE is solved using an Euler scheme with step size $\tau=T/N$, then under mild smoothness and boundedness conditions on the latent dynamics,
the translation error between the predicted and target images is bounded, with probability at least $1-\delta$, by
\begin{equation}
\label{Eq:ProbErrorBoundMain}
\|\vx_0 - \bar{\vx}_0\|_2
\le
W(T)\|\Delta(T)\|_2
+ \sqrt{\tfrac{\etW \etE}{\delta}}
+ \etC\tau^2
+ \varepsilon_{D_\varphi} ,
\end{equation}
where each term corresponds respectively to the encoder alignment, vector-field approximation, discretization, and decoder reconstruction errors.
\vspace{-1em}
\end{theorem}

\smallskip
\noindent
Note that $\etW$, $\etE$, and $\etC$ are constants and detailed derivation of them are provided in the Appendix. Theorem~\ref{prop:translation_error_main} explicitly incorporates the encoder alignment error rather than presuming it to be negligible. We minimize this error via structure-preserving fine-tuning and empirically validate the bound on challenging MRI-to-CT translation tasks in Section~\ref{Subsec:medcial_DINO} (Fig.~\ref{fig:medical_encoders}), showing that our method significantly reduces the gap compared to baselines and maintains robust translation performance across a broad range of initial semantic misalignment due to encoder imperfections.
\begin{table}[t]
\caption{Quantitative comparison on \textbf{MRI$\rightarrow$CT} (in-domain and out-of-domain (OOD)). (*) denotes comparison to pre-registered paired CT images. \textbf{Best} and \underline{second-best} values are highlighted among methods trained \emph{without} ground-truth supervision. The \textcolor{gray}{$\text{I}^2$SB} and \textcolor{gray}{SelfRDB} are supervised baselines.} 
\centering
\renewcommand{\arraystretch}{1.2}
\setlength{\tabcolsep}{4pt}
\resizebox{0.5\columnwidth}{!}{
\begin{tabular}{l|cc|cc}
\hline
\multirow{2}{*}{\textbf{Method}} & \multicolumn{2}{c|}{\textbf{MRI$\rightarrow$CT}} & \multicolumn{2}{c}{\textbf{MRI$\rightarrow$CT (OOD)}} \\
 & \textbf{MS-SSIM*} $\uparrow$ & \textbf{PSNR*} $\uparrow$ & \textbf{FID} $\downarrow$ & \textbf{MS-SSIM} $\uparrow$ \\
\hline
CycleGAN~\cite{zhu2017unpaired} & 0.657 & 18.86 & 127.22 & 0.514 \\
UNIT~\cite{liu2017unsupervised} & 0.712 & 20.88 & 115.84 & 0.458 \\
SDEdit~\cite{meng2022sdedit} & 0.628 & 18.87 & 50.17 & 0.385 \\
DDIB~\cite{su2022dual} & 0.625 & 18.45 & 48.15 & 0.411 \\
Syndiff~\cite{ozbey2023unsupervised} & \underline{0.805} & \underline{22.65} & 79.56 & 0.405 \\
DDBM~\cite{zhou2024denoising} & 0.768 & 22.09 & \underline{42.75} & \underline{0.540} \\
\textcolor{gray}{$\text{I}^2$SB~\cite{pmlr-v202-liu23ai}}  & \textcolor{gray}{0.776} & \textcolor{gray}{22.74}  & \textcolor{gray}{85.19}  & \textcolor{gray}{0.462} \\
\textcolor{gray}{SelfRDB~\cite{arslan2025self}} &  \textcolor{gray}{0.818} & \textcolor{gray}{23.01} & \textcolor{gray}{91.57} & \textcolor{gray}{0.393} \\
SSB(Ours) & \textbf{0.810} & \textbf{23.21} & \textbf{30.15} & \textbf{0.585} \\
\hline
\end{tabular}
}
\vspace{-1em}
\label{tab:mrict_translation}
\end{table}

\begin{table}[t]
\caption{Quantitative comparison on \textbf{Natural I2I} translations.}
\centering
\renewcommand{\arraystretch}{1.2}
\setlength{\tabcolsep}{3pt}
\resizebox{0.5\columnwidth}{!}{
\begin{tabular}{l|cccc}
\hline
\multirow{2}{*}{\textbf{Method}} & \multicolumn{4}{c}{\textbf{Horse$\rightarrow$Zebra / Apple$\rightarrow$Orange (256$\times$256)}} \\
 & \textbf{CLIP-T} $\uparrow$ & \textbf{LPIPS} $\downarrow$ & \textbf{SSIM} $\uparrow$ & \textbf{PSNR} $\uparrow$ \\
\hline
CycleGAN~\cite{zhu2017unpaired} & 0.225 & 0.416 & 0.630 & 15.57 \\
CUT~\cite{park2020contrastive} & 0.227 & 0.425 & 0.607 & 15.39 \\
SDEdit~\cite{meng2022sdedit} & 0.280 & 0.335 & 0.611 & 17.84 \\
DDIB~\cite{su2022dual} & 0.263 & 0.387 & 0.595 & 16.51 \\
CycleNet~\cite{xu2023cyclenet} & 0.305 & \underline{0.301} & \underline{0.719} & \underline{19.75} \\
ControlNet~\cite{zhang2023adding} & \textbf{0.327} & 0.475 & 0.426 & 14.43 \\
SSB (Ours) & \underline{0.322} & \textbf{0.285} & \textbf{0.794} & \textbf{20.20} \\
\hline
\end{tabular}
}
\vspace{-1em}
\label{tab:natural_translation}
\end{table}

\subsection{Practical Design Choices}
\label{Subsec:practial_design_choice}
We employ the following simple yet effective design choices to enhance controllability in structure-preserving translation and editing.
We interpolate between the source and target drifts, $v^{(j)}$ and $v^{(i)}$, using a time-varying coefficient $\eta_t > 0$ as
\begin{equation}
d\tilde{\vz_t}^{(i)} = d\vz_t^{(i)} + \eta_t(d\vz_t^{(j)} - d\vz_t^{(i)}).
\label{eq:rev-mixed-ode}
\end{equation}
This combination defines a smooth interpolating drift field that enables continuous domain transition within a shared latent manifold.
Similar to~\cite{meng2022sdedit,kulikov2024flowedit,rout2025semantic}, we parameterize $\eta_t = (1-t) \mathbb{I}[t > t_{\text{end}}]$, which provides a gradual trade-off between structural preservation and appearance adaptation. During early reverse steps, smaller $\eta_t$ promotes semantic consistency under high noise, while the cutoff $t_{\text{end}}$ relaxes this constraint—allowing flexible appearance modulation for translations with large domain or structural gaps, as illustrated in Fig.~\ref{fig:my_ablation2}. See Appendix for more details.
\begin{figure*}[t!] 
    \centering
    \includegraphics[width=1.\linewidth]{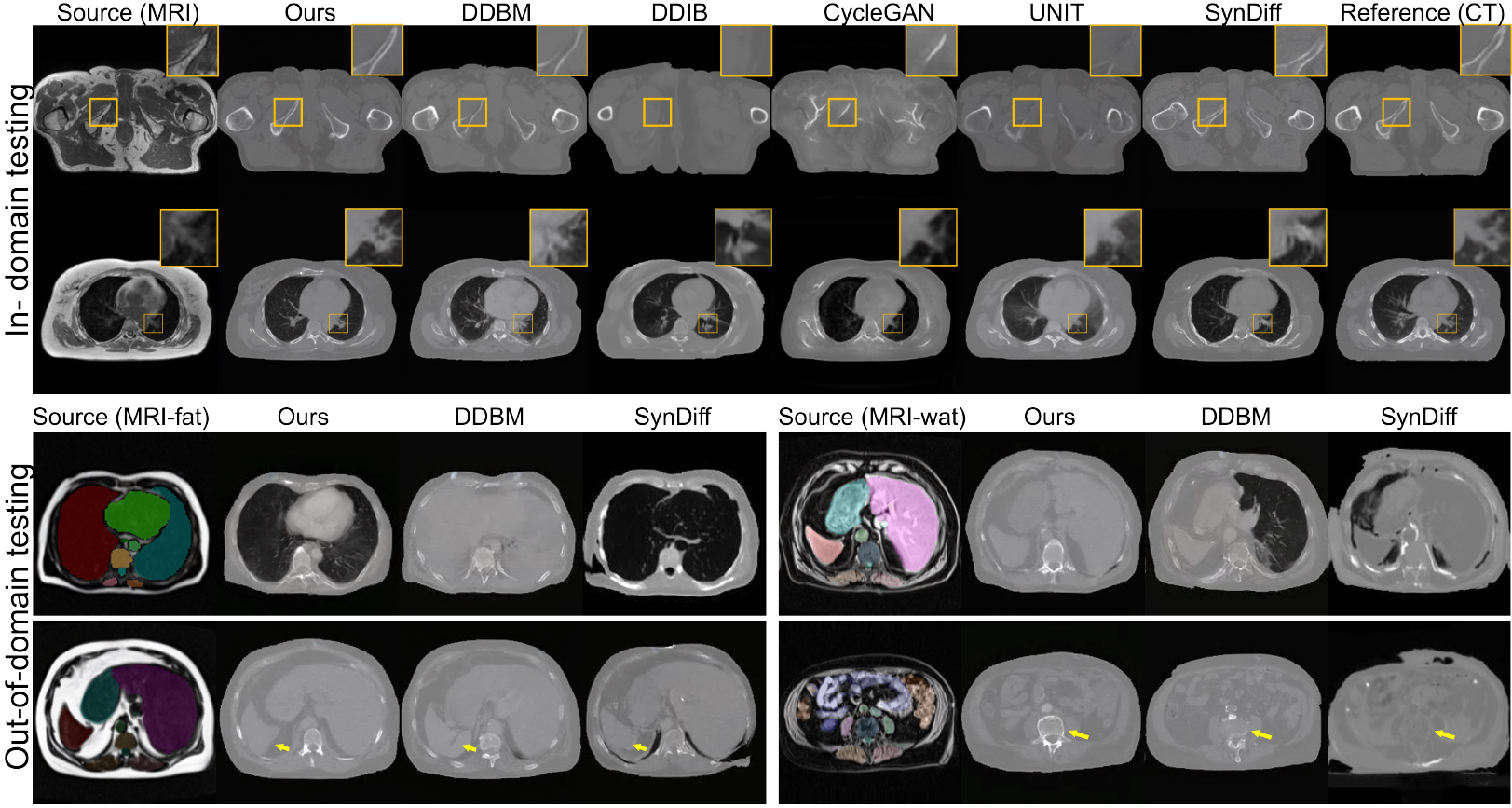}
    \vspace{-0.5em}
    \caption{Our SSB ensures anatomically consistent MRI$\rightarrow$CT translation across in-domain and out-of-domain scenarios. Segmentation masks are overlaid on MRI source images in OOD settings to provide structural reference, as paired CT ground truth is unavailable.
They are not used during training or inference and serve only to illustrate anatomical fidelity without segmentation supervision.}
\vspace{-0em}    
    \label{fig:main_results_medical_1}
\end{figure*}

\section{Experiments}
\label{Sec:experiments}
In this section, we present experimental comparisons of our method. For fairness, all baselines use their official implementations when publicly available, and detailed configurations are provided in the Appendix due to space limitations.
\subsection{Experimental Setup}
\label{SubSec:experimental_setup}
\noindent
\textbf{Medical MRI–CT Datasets.}
We conduct experiments on axial-view MRI$\rightarrow$CT translation with all images resized to $256\times256$ for efficiency.
For \textit{in-domain evaluation}, we use SynthRAD2023~\cite{thummerer2023synthrad2023} and SynthRAD2025~\cite{thummerer2025synthrad2025}, comprising $\sim950$ pre-registered MRI–CT volumetric pairs. Each data set is divided into training sets and test sets without subject overlap, yielding $6{,}596$ 2D test slices from $50$ subjects and the rest for training.
For \textit{training} our own shared-latent encoder and bridge models, we additionally aggregate unpaired MRI/CT data from AMOS2022~\cite{ji2022amos}, TotalSegmentator-MRI/CT~\cite{wasserthal2023totalsegmentator}, IXI~\citep{ixi_dataset}, pelvic MRI–CT~\cite{nyholm2018mr}, PSMA-FDG-PET-CT~\cite{gatidis2022whole}, and a subset of CT-RATE~\cite{hamamci2024developing}, augmented with sagittal and coronal views.
This yields about $1.05$M MRI and $2.2$M CT slices in total. For \textit{out-of-domain evaluation}, we sample $5{,}000$ fat- and water-contrast slices from $100$ UKBB-MRI~\cite{littlejohns2020uk} subjects, respectively, differing in contrast and resolution from the MRI training datasets.

\smallskip
\noindent
\textbf{General Domain Datasets.}
We evaluate our method on both class-label and text-guided I2I tasks. 
For class-label translation, we adopt Horse$\rightarrow$Zebra and Apple$\rightarrow$Orange~\cite{zhu2017unpaired}, using $30$ images per domain ($60$ total).  Each sample is paired with a GPT-5–generated target prompt to evaluate textual authenticity with respect to the target domain. For text-guided I2I editing, following~\cite{kulikov2024flowedit, tumanyan2022splicing}, we curate a diverse dataset for scene-style transfer (\emph{e.g.}, summer$\rightarrow$fall/winter) and object-level editing.
Images are sourced from Flickr~\citep{flickr}, DIV2K~\cite{agustsson2017ntire}, and Pexels~\citep{Pexels}, each paired with GPT-5–generated source captions and target prompts. In total, we collect $35$ natural scene images, synthesize $15$ additional scenes using FLUX-dev.~\cite{labs2025flux1kontextflowmatching}, and gather $60$ real object images, yielding over $100$ scene-style and $200$ object-editing text–image pairs.

\noindent
\textbf{Evaluation Metrics.}
For medical image translation (MRI$\rightarrow$CT), we evaluate image quality using normalized PSNR~\cite{gupta2018cnn}, multi-scale SSIM (MS-SSIM)~\cite{wang2003multiscale}, and FID~\cite{heusel2017gans}.
For natural I2I translation, which requires balancing semantic and structural consistency, we report CLIP-Text similarity (CLIP-T)~\cite{radford2021learning,rout2025semantic, kulikov2024flowedit, xu2023cyclenet}, DINO similarity~\cite{caron2021emerging,kulikov2024flowedit}, CLIP-Image similarity (CLIP-I), LPIPS~\cite{zhang2018perceptual}, MS-SSIM, and PSNR.
Similarly, for text-guided image editing, we evaluate semantic alignment with CLIP-T and structural fidelity with DINO, CLIP-I, LPIPS, SSIM~\cite{wang2004image}, and PSNR.

\begin{figure*}[t] 
    \centering
    \includegraphics[width=1.\linewidth]{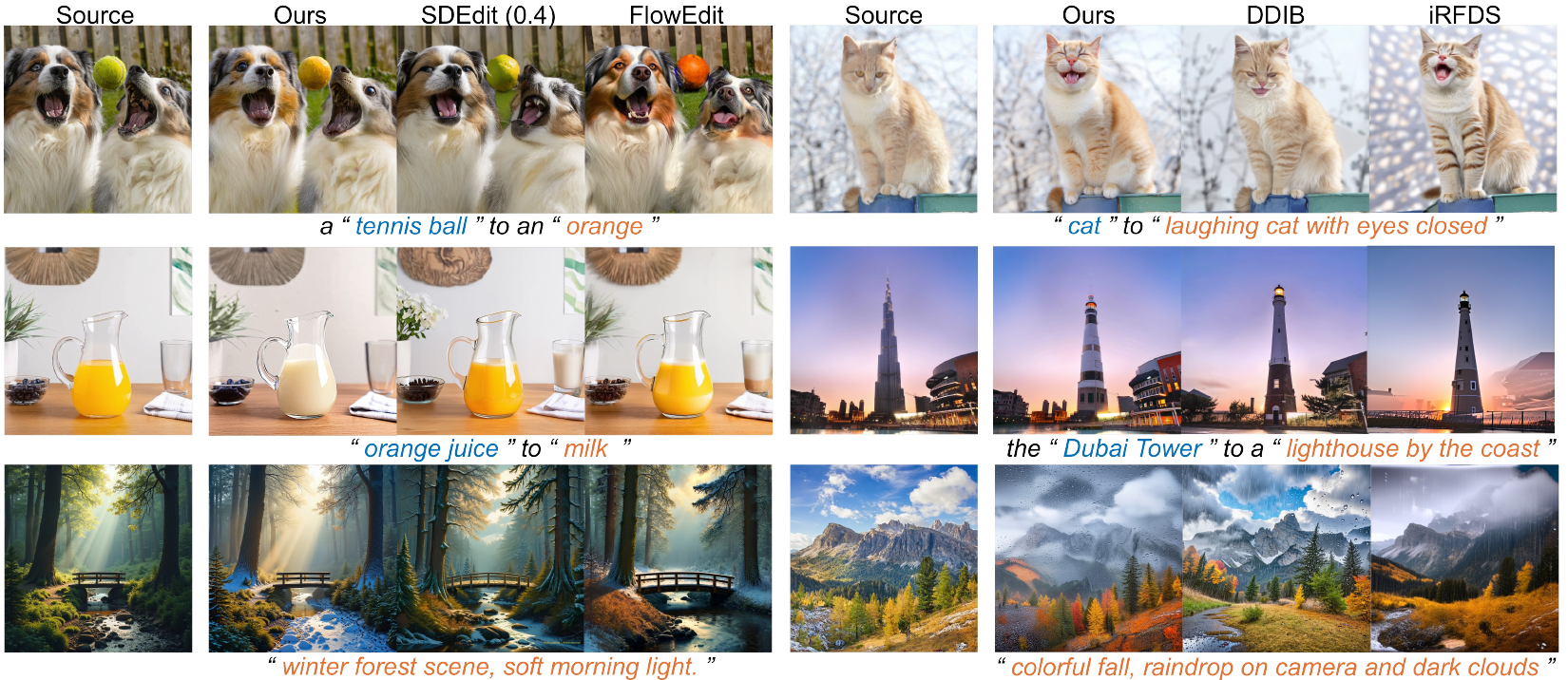}
    \vspace{-0.5em}
    \caption{Visual comparison with recent diffusion editing methods using SD3-M base model. Our approach generates visually consistent results with the source image, preserving structural integrity while convincingly applying appearance changes across diverse scenarios. }
\vspace{-1.em}    
    \label{fig:main_results_edit_1}
\end{figure*}
\begin{figure*}[t!] 
    \centering
    \includegraphics[width=0.9\linewidth]{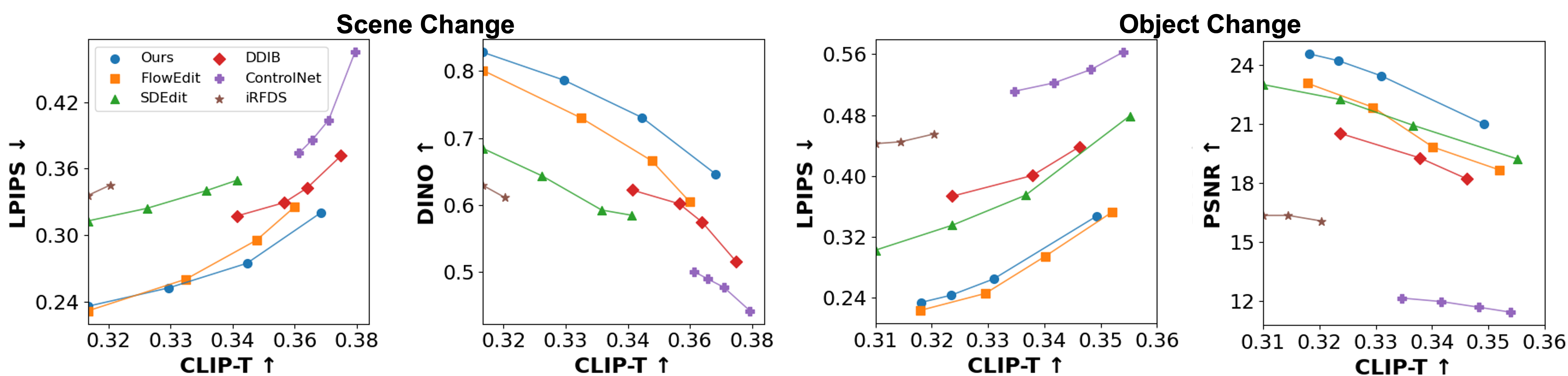}
    \vspace{-0.5em}
    \caption{Numerical results based on SD3-M model. Our method achieves a favorable balance between text adherence (CLIP-T) and structural preservation (DINO and PSNR) across both scene-editing and object-change tasks. Scene editing requires substantial global modifications, whereas object change focuses on localized adjustments. Connected markers represent different hyperparameter settings.}
    \vspace{-0em}
    \label{fig:table2}
\end{figure*}

\subsection{Appearance-Invariant Encoder Validation}
\label{Subsec:medcial_DINO}

We begin by validating the encoder alignment assumption from Theorem~\ref{prop:translation_error_main} via feature analysis on the MRI–CT dataset. Using CKNNA~\citep{huh2024position} for local structure and Cosine Similarity for global alignment, we find that standard DINOv2 exhibits a trade-off: acceptable alignment but high-frequency structural noise (see Fig.~\ref{fig:medical_encoders}), indicating a gap from the ideal geometry-preserving manifold.To minimize the encoder alignment error ($\|\Delta(T)\|_2$), we fine-tune a DINOv2-ViT-B/8 on the MRI–CT training set. Crucially, we integrate a retina-inspired filter~\cite{8306443} to suppress modality-specific appearance (e.g., contrast) and force the model to rely on structure representations. As shown in Fig.~\ref{fig:medical_encoders} (Left), our encoder achieves the optimal balance of metrics, significantly reducing the deviation from the ideal manifold compared to baselines like BiomedCLIP~\cite{zhang2023biomedclip}, MedSigCLIP~\cite{sellergren2025medgemma}, DINOv2-B~\citep{oquab2024dinov}, and MedVAE~\cite{varma2025medvae}. Qualitative visualizations (Right) confirm that our method eliminates patch-grid artifacts, yielding the smooth, anatomically consistent feature space required to satisfy our theoretical error bound. To the best of our knowledge, this is the first adaptation of DINOv2 specifically for MRI–CT image translation.

\subsection{Unpaired Image Translation}
\label{SubSec:unpaired_image_translation}
\noindent
\textbf{Medical I2I.} Given the pre-trained DINOv2 ViT-B/8 encoder described in Sec.~\ref{Subsec:medcial_DINO}, we construct the latent bridge using a U-Net backbone~\cite{dhariwal2021diffusion} between the DINOv2 encoder and a KL-regularized VAE~\cite{rombach2022high}, producing latent feature maps with a spatial resolution of $64\times64$ and $8$ channels. Since our focus is MRI$\rightarrow$CT synthesis, we train a single latent bridge model on the CT training datasets. Following~\citep{karras2022elucidating, zhou2024denoising, zhang2024exploring}, we parameterize the interpolant weights defined in Eq.~\ref{eq:latent-SI-forward-kernel} in terms of the signal-to-noise ratio (SNR), defined as $\mathrm{SNR}_t = a_t^2 / \sigma_t^2$ with $(a_t,\sigma_t)$ the standard DDPM~\citep{ho2020denoising} schedules. The weights are $\alpha_t = a_t\!\left(1-\frac{\mathrm{SNR}_T}{\mathrm{SNR}_t}\right)$, $\beta_t = \frac{a_t}{a_T}\,\frac{\mathrm{SNR}_T}{\mathrm{SNR}_t}$, $\gamma_t^{2} = \sigma_t^{2}\!\left(1-\frac{\mathrm{SNR}_T}{\mathrm{SNR}_t}\right)$. At inference, we adapt the DDBM~\cite{zhou2024denoising} hybrid sampler which is based on the predictor-corrector sampler introduced in~\citep{song2020score}.

We compare our method against representative baselines widely used in medical I2I translation, including GAN-based models CycleGAN~\cite{zhu2017unpaired} and UNIT~\cite{liu2017unsupervised}, zero-shot diffusion approaches SDEdit~\cite{meng2022sdedit} and DDIB~\cite{su2022dual}, and the hybrid CycleGAN–diffusion framework SynDiff~\cite{Muzaffer2023tmi}.
Since edge-based representations, such as Canny edges, are commonly used in MRI–CT translation—where diffusion models learn to reconstruct images from edge maps in a self-supervised manner~\cite{chen2024contourdiff}—we additionally implement a DDBM variant using Canny edges. This comparison demonstrates that DINO-based embeddings capture richer geometric and semantic information than handcrafted edge filters. Fig.~\ref{fig:main_results_medical_1} and Table~\ref{tab:mrict_translation} present qualitative and quantitative results on in-domain and OOD MRI$\rightarrow$CT translation.  Since no paired CT is available, we compute FID against CT scans from the training set and MS-SSIM between the input MRI and the synthesized CT as a proxy for structural similarity. We also provide qualitative results and segmentation overlays, with additional visuals in the Appendix. Our method shows stronger robustness to new MRI contrasts and achieves more accurate translations than SynDiff and DDBM, preserving geometry and modality realism. 

\noindent
\textbf{General Domain I2I Translation.} We compare our method against established baselines on the Horse $\rightarrow$ Zebra and Apple $\rightarrow$ Orange benchmarks. For this task, we fine-tune the flow-based SiT transformer~\cite{ma2024sit,yu2024representation}, pre-trained on ImageNet1K, at 256×256 image resolution. Since recently released DINOv3~\citep{simeoni2025dinov3} provides higher-quality dense features than official DINOv2, we adopt the pretrained DINOv3 ViT-L/16 as our backbone for general-domain tasks. As discussed in Sec.~\ref{Subsec:practial_design_choice}, we inject PCA-compressed DINOv3 features (first-$B{=}16$) into the SiT backbone via a zero-initialized linear projection layer with positional embeddings. The projected features are added in parallel to the SiT input layer for feature fusion, enabling smooth initialization when introducing new conditioning inputs, similar to the zero-start strategy in~\cite{zhang2023adding}. For endpoint $\vz_T$, the projected features are channel-wise averaged to match the latent dimensionality of the KL-VAE endpoint $\vz_0$. We adopt the linear weights for the endpoints and a bridge variance schedule $\alpha_t = 1-t$, $\beta_t = t$, $\gamma_t^2 = \gamma_{\max}^2\,t(1-t)$, with $\gamma_{\text{max}}=0.1$. We set $b=1$, resulting $\vz_T \sim\gN\big(E_{\phi}(\vx_0), \mathbf{I})$. The entire SiT is then finetuned on the same ImageNet1k-256 dataset used for pretraining.
Table~\ref{tab:natural_translation} presents the average numerical comparisons with baselines including CycleGAN~\cite{zhu2017unpaired}, CUT~\cite{park2020contrastive}, SDEdit~\cite{meng2022sdedit}, DDIB~\cite{su2022dual}, CycleNet~\cite{xu2023cyclenet}, and ControlNet~\cite{zhang2023adding}. Overall, our method achieves the best combination of authenticity to the target text (CLIP-T) and structural consistency with the source.
\subsection{Text-to-Image Editing}
In this section, we present our experimental results on text-guided I2I editing. We finetune the pre-trained text-image SD3~\cite{esser2024scaling} medium (SD3-M) model. Similar to SiT fine-tuning, we inject first-$B{=}32$ PCA-compressed DINOv3 features into SD3’s hidden layers via zero-initialized projections, and construct $\vz_T$ by channel-wise averaging across the PCA components. We use the same endpoint weighting and bridge variance schedules for the unpaired I2I translation present before. We collect a small yet high-resolution subset of the LAION-5B~\cite{schuhmann2022laion} dataset to fine-tune our SD3-M model. This subset contains approximately 1.2 M text–image pairs, with an average resolution of about 1200 × 1400 pixels. For computational efficiency, we update only the attention layers of the model while keeping all original linear layers frozen (see additional implementation details in the Appendix). We compare our method with recent baselines officially implemented on SD3-M, including SDEdit~\cite{meng2022sdedit}, DDIB~\cite{su2022dual}, iRFDS~\cite{yang2025texttoimage}, FlowEdit~\cite{kulikov2024flowedit}, and ControlNet~\cite{zhang2023adding}. Because semantic alignment (CLIP-T) and structural fidelity often trade off~\citep{kulikov2024flowedit}, we evaluate each method across a range of hyperparameters to reveal their operating tradeoff curves in Fig.~\ref{fig:table2}. The left panel shows scene-style editing, which involves substantial global appearance changes. Accordingly, structural fidelity for scene-style editing metrics (SSIM, PSNR, LPIPS) are computed on the luminance (YCbCr) channel to better reflect perceptual structure consistency. Fig.~\ref{fig:table2} (\emph{Right}) shows object-level editing with localized modifications. Our method achieves comparable or superior performance to state-of-the-art approaches such as FlowEdit across both semantic and structural metrics, particularly in complex scene edits that demand significant style shifts. Fig.~\ref{fig:my_picture1} presents visual comparisons with ControlNet.
\begin{figure}[t!] 
    \centering
    \includegraphics[width=1.\linewidth]{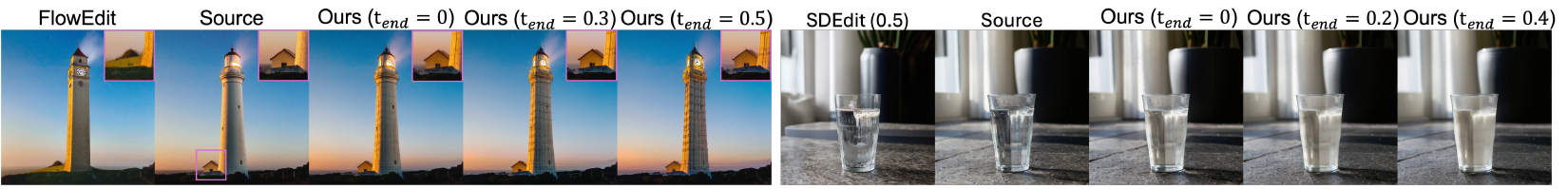}
    \vspace{-0.0em}
    \caption{Balancing structure and appearance via vector-field control. Interpolating between domain flows (~\emph{e.g., Left: “lighthouse” → “clock tower”} and \emph{Right: “water” → “milk”}) enables controllable trade-offs between structural consistency and appearance fidelity.}
    \vspace{-1em}
    \label{fig:my_ablation2}
\end{figure}

\begin{figure}[t] 
    \centering
    \includegraphics[width=0.5\linewidth]{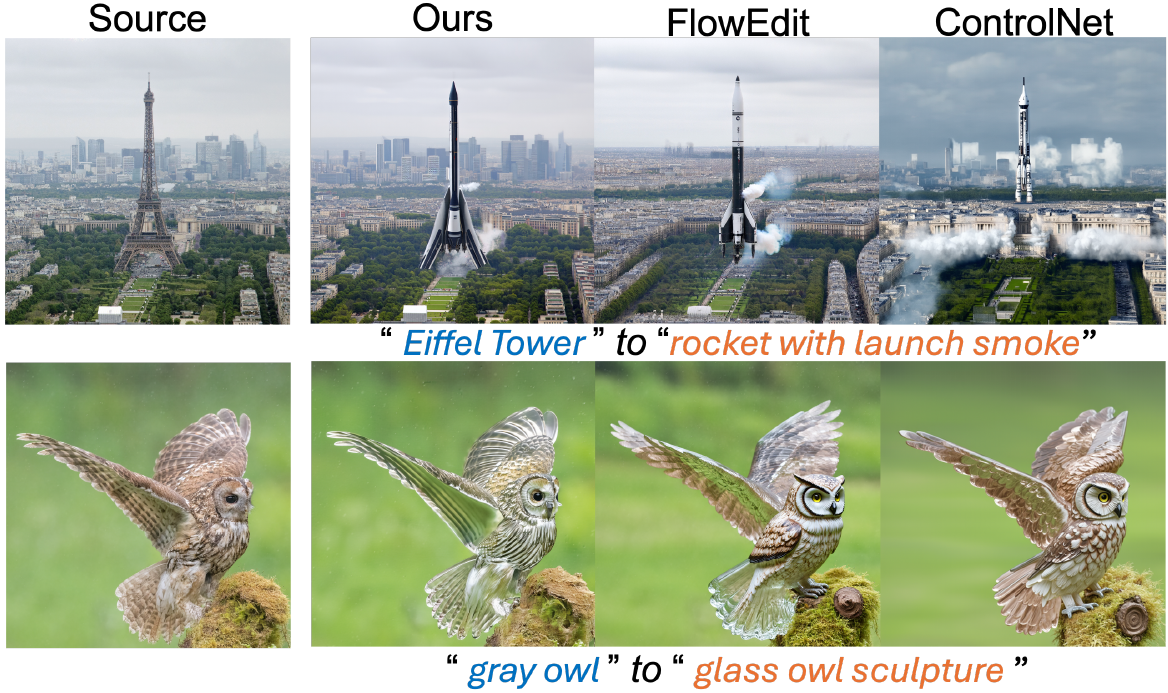}
    \vspace{-0em}
    \caption{More results comparing with FlowEdit and ControlNet.}
    \vspace{-1em}
    \label{fig:my_picture1}
\end{figure}

\section{Conclusion and Limitations}
We introduce the SSB, a diffusion-based framework for unpaired image translation and editing built upon a shared latent-space formulation.
By constructing a unified latent space across domains through self-supervised, appearance-invariant encoders, SSB enables structure-preserving and semantically aligned translation without paired supervision or adversarial training.
Experiments across both medical and natural domains demonstrate that this unified latent representation facilitates robust cross-domain generalization, yielding anatomically consistent MRI→CT synthesis and controllable appearance manipulation in natural images. While SSB’s strong structure preservation enables high-fidelity appearance changes, it becomes less effective when transformations require fundamentally altering object geometry. Edits that imply a shift in object category violate the geometry-aware prior underlying our semantic manifold. We discuss and illustrate these failure cases in the Appendix.

\section*{Acknowledgments}
This work was supported by the Stanford Center for AI in Medicine and Imaging (AIMI) and the Stanford Institute for Human Centered AI (HAI).

\bibliographystyle{unsrt}  
\bibliography{references}  
\newpage
\appendix

\begin{center}
    {\huge \bf Supplementary Material for \\ Self-Supervised Semantic Bridge \par}
    \vspace{1em}
\end{center}

\section{Theoretical Details}
Our theoretical analysis aims to quantify the difference between the ground-truth target $\vx_0$ and the predicted translation $\vxbar_0$. We consider the ODE formulations of diffusion models. The following transparent assumptions are made throughout our analysis.

\begin{assumption}
\label{As:Encoder}
Given a pair of source and target images $(\vx^{(j)}, \vx^{(i)})$ that share the same latent code $\vz$, the learned encoder $E_\phi : \R^n \rightarrow \R^m$ may not perfectly map both images to the identical latent code; that is, $E_\phi(\vx^{(j)}) \neq E_\phi(\vx^{(i)})$.
\end{assumption}

\noindent
This assumption accounts for the imperfection of the learned encoder. In our analysis, we address this by introducing an error term associated with different initializations.

\begin{assumption}
\label{As:Decoder}
We assume the learned KL-VAE decoder $D_\varphi:\R^m\rightarrow\R^n$ satisfies the following conditions:
\begin{enumerate}[ref=\theassumption.(\alph*), label=(\alph*)]
\item\label{As:Decoder.a} $D_\varphi$ has a bounded reconstruction error; that is, $\|D_\varphi(\vz) - \vx\|_2\leq\varepsilon_{D_\varphi}$, where $\vz$ denotes the latent code of $\vx$.
\item\label{As:Decoder.b} $D_\varphi$ is Lipschitz continuous with $L_{D_\varphi}>0$ for any $\vz_1, \vz_2\in\R^m$
\begin{equation*}
\|D_\varphi(\vz_1) - D_\varphi(\vz_2)\|_2 \leq L_{D_\varphi} \|\vz_1 - \vz_2\|_2.
\end{equation*}
\end{enumerate}
\end{assumption}
\noindent
Assumption~\ref{As:Decoder.a} accounts for the imperfection of the learned decoder by assuming a bounded reconstruction error.
Assumption~\ref{As:Decoder.b} imposes a standard regularity condition to ensure that the output of decoder does not change drastically.

\begin{assumption}
\label{As:Vector}
We assume the learned latent vector field $v_\theta(t, \vz, \vz_T)$ satisfies the following conditions:
\begin{enumerate}[ref=\theassumption(\alph*), label=(\alph*)]
\item \label{As:Vector.a} For any $t>0$, $v_\theta(t, \vz, \vz_T)$ is Lipschitz continuous with $L_v(t)>0$ for any $\vz_1, \vz_2\in\R^n$
\begin{equation*}
\|v_\theta(t, \vz_1, \vz_T) - v_\theta(t, \vz_2, \vz_T)\|_2 \leq L_v(t) \|\vz_1 - \vz_2\|_2.
\end{equation*}
\item \label{As:Vector.b} For any $t>0$, $v_\theta(t, \vz, \vz_T)$ has a $L^2$-accurate error
\begin{equation*}
\E\left[ \|v_\theta(t, \vz, \vz_T) - v(t, \vz, \vz_T)\|_2^2 \right] \leq \varepsilon_v(t),
\end{equation*}
where $\varepsilon_v(t)<+\infty$.
\end{enumerate}
\end{assumption}

\noindent
Note that our analysis only requires $L^2$-accurate vector field estimate under the measure $p_t$, which is a relatively weaker condition than the $L^\infty$-accurate estimate that assumes the error between the learned and ground-truth score is uniformly bounded.

\subsection{Derivation for Theorem 4.1}
\label{Sec:DerivationODE}

In this section, we present the detailed derivation for Theorem 4.1. 
The derivation is organized into four steps. The first three steps progressively quantify the errors introduced by the encoder, the learned vector field, the Euler discretization, and the decoder, while the final step derives the high-probability upper bound.

\smallskip\noindent
\textbf{Step 1: Encoder Error \& Continuous-Time Error Propagation.} 
Consider the following ODEs characterized by ideal and learned vector fields, respectively
\begin{align}
\label{Eq:ODEs}
&d\vz_t = v(t, \vz_t, \vz_T) dt,\quad \text{with} \quad \vz_T = E_\phi(\vx^{(i)})\\ 
&d\vzhat_t = v_\theta(t, \vzhat_t, \vzhat_T) dt,\quad \text{with} \quad \vzhat_T = E_\phi(\vx^{(j)})
\end{align}
where both ODEs evolve from $T$ to $0$.
Define the time-varying difference as $\Delta(t) = \vz_t - \vzhat_t$. We can write
\begin{align}
\label{Eq:DifferenceODE}
\frac{d}{dt}\Delta(t) &= v(t, \vz_t, \vz_T)- v_\theta(t, \vzhat_t, \vzhat_T) \nonumber\\
&= \Big(v(t, \vz_t, \vz_T)- v_\theta(t, \vz_t, \vz_T)\Big) + \Big(v_\theta(t, \vz_t, \vz_T)- v_\theta(t, \vzhat_t, \vzhat_T)\Big)
\end{align}
Taking the norm yields
\begin{align*}
\|\frac{d}{dt}\Delta(t)\|_2 
&\leq \|v_\theta(t, \vz_t, \vz_T)- v_\theta(t, \vzhat_t, \vzhat_T)\|_2 + \|v(t, \vz_t, \vz_T)- v_\theta(t, \vz_t, \vz_T)\|_2 \\
&= L_v(t)\|\Delta(t)\|_2 + \|\ve(t)\|_2.
\end{align*}
Here, we define $\ve(t)=v(t, \vz_t, \vz_T)- v_\theta(t, \vz_t, \vz_T)$.
We can bound the growth of $\|\Delta(t)\|_2$ by using the Cauchy-Schwartz inequality
\begin{align*}
\frac{d}{dt}\|\Delta(t)\|_2 
&= \frac{\Delta(t)^\Tsf \frac{d}{dt}\Delta(t)}{\|\Delta(t)\|_2} 
\leq \frac{\|\Delta(t)\|_2\cdot \|\frac{d}{dt}\Delta(t)\|_2}{\|\Delta(t)\|_2} \leq L_v(t)\|\Delta(t)\|_2 + \|\ve(t)\|_2.
\end{align*}
Due to the imperfection of the learned encoder, the two ODEs are run with different initialization; that is, $\|\Delta(T)\|_2 = \| E_\phi(\vx^{(j)}) - E_\phi(\vx^{(i)}) \|_2 \neq 0$.
To quantify the difference at end point $t=0$, we use the \textit{Grönwall Theorem} (see Theorem~\ref{STh:Gronwall}) with time-varying coefficients
\begin{align}
\label{Eq:ContinuousODE}
\|\Delta(0)\|_2 
&\leq \exp\bigg(\int_0^T L_v(s) ds\bigg) \|\Delta(T)\|_2 + \int_0^T \exp\bigg(\int_0^t L_v(s) ds\bigg) \|\ve(t)\|_2 dt
\end{align}
Equation \eqref{Eq:ContinuousODE} relates $\|\Delta(0)\|_2$ to the vector field approximation error in the continuous-time domain.

\smallskip\noindent
\textbf{Step 2: Euler Discretization Error.} 
We next incorporate the error caused by discretizing the ODE. 
Here, we consider the \textit{Euler method}, as it is the foundation of other popular first-order solvers. 
Our analysis builds on classical analysis on the error of ODE discretization~\citep{butcher2016numerical}, and below we provide a brief derivation for the reader's convenience. To begin with, we introduce a standard assumption on the second-order time derivative of the solution path $\vzhat_t$~\citep{butcher2016numerical}.

\begin{assumption}
\label{As:TimeDerivative}
The second-order time derivative of the solution $\vzhat_t$ is bounded; that is, $\sup_{t\in[0,T]}\|\frac{d^2}{dt^2}\vzhat_t\|_2 \leq B$, where $B>0$ is a finite constant.
\end{assumption}

Consider the Euler method for solving the ODE in \Eqref{Eq:ODEs} for $t\in[0,T]$.
\begin{equation*}
\vzbar_{i+1} = \vzbar_{i} + \tau v_\theta(t_i, \vzbar_{t_i}, \vzhat_T),
\end{equation*}
where we use $\vzbar_i$ to denote the discrete solution, $\tau=T/N$ the stepsize, $N>0$ the number of steps, and $t_i$ the time with $t_i = T-i\tau$.
Applying the first-order Taylor expansion, we define the local truncation error for $[t_i, t_{i+1}]$
\begin{equation*}
\veta_i = \vzhat_{t_{i+1}} - \Big(\vzhat_{t_i} + \tau v_\theta(t_i, \vzhat_{t_i}, \vzhat_T)\Big).
\end{equation*}
Define the discretization error at step $i$ as $\varepsilon_i^\text{disc} = \|\vzhat_{t_i} - \vzbar_i\|_2$. We can write
\begin{align}
\label{Eq:Recursion}
&\varepsilon_{i+1}^\text{disc} = \|\vzhat_{t_{i+1}} - \vzbar_{i+1}\|_2 \nonumber\\
&= \left\|\Big( \vzhat_{t_i} + \tau v_\theta(t_i, \vzhat_{t_i}, \vzhat_T) + \veta_i \Big) - \Big( \vzbar_{i} + \tau v_\theta(t_i, \vzbar_i, \vzhat_T) \Big) \right\|_2 \nonumber\\
&= \left\|\vzhat_{t_i} - \vzbar_i + \tau \Big( v_\theta(t_i, \vzhat_{t_i}, \vzhat_{T}) - v_\theta(t_i, \vzbar_i, \vzhat_T) \Big) + \veta_i \right\|_2 \nonumber\\
&\leq (1+\tau L_v(t))\varepsilon_i^\text{disc} + \|\veta_i\|_2
\end{align}
The above inequality establishes the multiplicative recursion for $\varepsilon_i$ with the truncation error $\|\veta_i\|_2$.
From Assumption~\ref{As:TimeDerivative}, we know that the truncation error is bounded 
\begin{equation*}
\|\veta_i\|_2 \leq \frac{\tau^2}{2} \|\frac{d^2}{dt^2}\vz_{\xi_i} \|_2 \leq \frac{\tau^2}{2}B,
\end{equation*}
where $\xi_i\in(t_i, t_{i+1})$. As we run the continuous-time and discretized ODEs with the same initialization ($\vzbar_0 = \vzhat_{t_0} = \vzhat_T$), the initial discretization error is zero ($\varepsilon_0^\text{disc}=0$).
By unrolling~\eqref{Eq:Recursion}, we obtain
\begin{align*}
\varepsilon_N^\text{disc} &= \|\vzhat_0 - \vzbar_N \|_2 \\
&\leq \underbrace{\left( \prod_{j=0}^{N-1} (1+\tau L_v(t_j)) \right) \varepsilon_0^\text{disc}}_{=0} + \left( \frac{B}{2}\sum_{i=0}^{N-1} \prod_{j=i+1}^{N-1} (1+\tau L_v(t_j)) \right) \tau^2,
\end{align*}
We can simplify the product in the nonzero term by using facts $1+u\leq e^u$
\begin{equation*}
\prod_{j=i+1}^{N-1} (1+\tau L_v(t_j)) \leq \exp\left(\sum_{j=i+1}^{N-1}\tau L_v(t_j)\right) 
\end{equation*}
Expressing the summation as an integral, we obtain
\begin{align}
\label{Eq:DistError}
\varepsilon_N^{\text{disc}}
&\leq 
\left(\frac{B}{2} 
\sum_{i=0}^{N-1}
\prod_{j=i+1}^{N-1} 
(1+\tau L_v(t_j))
\right)\tau^2 
\leq 
\left(\frac{B}{2}
\sum_{i=0}^{N-1}
\exp\!\left(
\sum_{j=i+1}^{N-1}\tau L_v(t_j)
\right)
\right)\tau^2.
\end{align}

\smallskip\noindent
\textbf{Step 3: Decoder Error \& Deterministic Upperbound.} Now we derive the final bound by accounting the imperfection of the learned decoder. According to Assumption~\ref{As:Decoder}, we can write

\begin{equation*}
\begin{aligned}
\|\vx - \bar{\vx}\|_2
&= \big\|
D_\varphi(\vz_0) - D_\varphi(\bar{\vz}_N)
+ \vx - D_\varphi(\vz_0)
\big\|_2 \\
&\leq L_{D_\varphi} \, \|\vz_0 - \bar{\vz}_N\|_2 + \varepsilon_{D_\varphi}.
\end{aligned}
\end{equation*}
Applying triangle inequality yields to
\begin{align}
\label{Eq:DiffBound}
\|\vz_0 - \vzbar_N\|_2 
&\leq \| \vz_0 - \vzhat_0 \|_2 + \| \vzhat_0 - \vzbar_N \|_2 \nonumber\\
&= \|\Delta(0)\|_2 + \varepsilon_N.
\end{align}
Combining equations \eqref{Eq:ContinuousODE} and \eqref{Eq:DistError}, we can obtain the final bound
\begin{align}
\label{Eq:Bound}
\|\vx - \bar{\vx}\|_2 
&\leq 
\underbrace{\vphantom{\int_0^T} W(T)\|\Delta(T)\|_2}_{\text{Encoder Err.}}
+ \underbrace{\int_0^T W(t)\|\ve(t)\|_2\,dt}_{\text{Field Approx. Err.}} 
+ \underbrace{\vphantom{\int_0^T}\etC\tau^2}_{\text{Discret. Err.}}
+ \underbrace{\vphantom{\int_0^T}\varepsilon_{D_\varphi}}_{\text{Decoder Err.}}.
\end{align}

where $W(t)$ and $\etC$ are given by
\begin{align}
\label{Eq:WandC}
W(t) &= L_{D_\varphi}\exp\!\left(\int_0^t L_v(s)\,ds\right), \nonumber\\
\etC &= \frac{B}{2}\sum_{i=0}^{N-1}\exp\!\left(\sum_{j=i+1}^{N-1}\tau L_v(t_j)\right).
\end{align}
\noindent

\smallskip\noindent
\textbf{Step 4: Probabilistic Upperbound.} 
In~\eqref{Eq:Bound}, randomness arises from the vector field approximation error according to Assumption~\ref{As:Vector.b}.
For notational convenience, we define the following quantities
\begin{equation*}
\begin{aligned}
\etE &\triangleq \int_0^T \varepsilon_v(t)\,dt, \\
\etW &\triangleq \int_0^T W(t)^2\,dt, \\
\etQ &\triangleq \int_0^T \|\ve(t)\|_2^2\,dt.
\end{aligned}
\end{equation*}
According to the Cauchy-Schwartz inequality, we can obtain
\begin{align}
\label{Eq:CauchySchwartzWQ}
\int_0^T W(t)\,\|\ve(t)\|_2\,dt 
&\leq \sqrt{\int_0^T W(t)^2\,dt}\,
       \sqrt{\int_0^T \|\ve(t)\|_2^2\,dt} 
       \leq \sqrt{\etW\,\etQ}.
\end{align}
Applying the \textit{Markov inequality} (see Theorem~\ref{STh:Markov}) to $\etQ$ yields
\begin{equation*}
\Pr\left(\etQ\leq\frac{\etE}{\delta}\right) 
= 1-\Pr\left(\etQ\geq\frac{\etE}{\delta}\right) 
\geq 1-\frac{\E[\etQ]}{\etE/\delta}.
\end{equation*}
As $\|\ve(t)\|_2^2$ is nonnegative, we can swap the integral in $\E[\etQ]$ (Fubini-Tonelli Theorem) and write
\begin{align*}
\E[\etQ] 
&= \E\left[ \int_0^T\|v(t, \vz_t, \vz_T)- v_\theta(t, \vz_t, \vz_T)\|_2^2dt \right] \\
&= \int_0^T\E\left[ \|v(t, \vz_t, \vz_T)- v_\theta(t, \vz_t, \vz_T)\|_2^2 \right]dt \\
&\leq \int_0^T \varepsilon_v (t) dt = \etE,
\end{align*}
where the last equality uses Assumption~\ref{As:Vector.b}. Overall, we have 
\begin{equation}
\label{Eq:QProb}
\Pr\left(\etQ\leq\frac{\etE}{\delta}\right) 
\geq 1-\frac{\E[\etQ]}{\etE/\delta}
\geq 1-\frac{\etE}{\etE/\delta}
\geq 1-\delta.
\end{equation}
Note that, by applying~\eqref{Eq:CauchySchwartzWQ}, we can express~\eqref{Eq:Bound} as
\begin{equation*}
\|\vx-\vxbar\|_2 \leq W(T) \|\Delta(T)\|_2 + \sqrt{\etW\etQ} + C \tau^2 + \varepsilon_{D_\varphi}.
\end{equation*}
Applying the result in \eqref{Eq:QProb}, we derive the result in Theorem 4.1, that is, with probability at least $1-\delta$,
\begin{equation*}
\|\vx-\vxbar\|_2 \leq W(T) \|\Delta(T)\|_2 + \sqrt{\frac{\etW\etE}{\delta}} + \etC \tau^2 + \varepsilon_{D_\varphi}.
\end{equation*}

\subsection{Analysis of the DINO Encoders Training}

\begin{remark}[Appearance invariance]
\label{lem:appearance-invariance}
Let augmentations factor as $A = G \circ C$, 
with $C \in \mathcal C$ (appearance transforms,~e.g., color, contrast, and style) 
and $G \in \mathcal G$ (geometric transforms,~e.g., shape scale, and layout). 
At stationary ($\phi \approx \phi'$), any minimizer $\phi^\star$ of $\mathcal L{(\phi;\phi')}$ satisfies, for almost every $\vx$,
\[
p_{\phi^\star}(G \circ C(\vx)) \;\approx\; p_{\phi^\star}(G(\vx)), 
\quad \forall\, C \in \mathcal C,\; G \in \mathcal G.
\]
\end{remark}
\begin{proof}
At stationarity ($\phi=\phi'$), the training loss reduces to
\begin{equation}
\mathcal L_{\phi} 
= \mathbb E_{A_1,A_2,\vx}\!\left[
-\,p_{\phi}(A_2(\vx))^\top \log p_\phi(A_1(\vx))
\right].
\end{equation}
For fixed $(\vx,A_1,A_2)$ this is the cross-entropy 
$\mathrm{CE}(q,p)$ between $q=p_\phi(A_2(\vx))$ and $p=p_\phi(A_1(\vx))$. 
Since $\mathrm{CE}(q,p)\geq H(q)$ with equality iff $p=q$, the loss is minimized precisely when
\begin{equation}
p_\phi(A_1(\vx)) = p_\phi(A_2(\vx)), 
\quad \forall\, A_1,A_2.
\label{eq:CE-min}
\end{equation}
Writing each augmentation as $A=G\circ C$ with $C\in\mathcal C$ (appearance) and 
$G\in\mathcal G$ (geometry), condition~\eqref{eq:CE-min} implies
\begin{equation}
p_\phi(G\circ C(\vx)) = p_\phi(G(\vx)), 
\quad \forall\, C\in\mathcal C,\; G\in\mathcal G,
\end{equation}
which proves the remark.
\end{proof}

\subsection{Supporting Theorems}
We here summarize all the theorems used in our proofs.

\begin{theorem}[Grönwall Theorem]
\label{STh:Gronwall}
Let $f(t)$ be a nonnegative, absolutely continuous function on $[0,T]$, which satisfies for almost everywhere $t$ the differential inequality
$$\frac{df(t)}{dt}\leq c(t) f(t) + h(t),$$
where $f(t)$ and $h(t)$ are nonnegative, summable functions on [0,T]. Then
$$f(t)\leq \exp\left(\int_0^t c(s) ds\right)\left[f(0) + \int_0^t h(s)ds\right].$$
\end{theorem}
\begin{proof}
See Appendix B in~\cite{evans2022partial}.
\end{proof}

\begin{theorem}[Markov Inequality]
\label{STh:Markov}
Let $X$ be a nonnegative random variable and let $a>0$. Then, we have
$$\Pr(X \geq a) \leq \frac{\E[X]}{a}.$$
\end{theorem}
\begin{proof}
See pp 276 in~\cite{ash2000probability}.
\end{proof}

\section{Datasets Usage}
\label{Sec:Dataset}
This section provides comprehensive information about the datasets used in our experiments, including data characteristics, annotation details, and their roles in our experimental setup. We organize datasets into two categories: Medical images and Natural images.

\noindent
\subsection{Medical Images}
When constructing the CT training and test sets for all methods in this work, we apply TotalSegmentator~\cite{wasserthal2023totalsegmentator} to each CT volume to obtain a body mask and discard non-body regions. This removes the scanner table and other background objects, making CT slices visually closer to MRI and reducing the cross-modality appearance gap. This preprocessing step relies on an off-the-shelf tool and is applied only once per scan. It is therefore relatively cheap and unlikely to become a computational bottleneck.  

It is worth noting that, compared with large-scale natural image collections (e.g., LAION-5B~\cite{schuhmann2022laion}), publicly available MRI/CT datasets are much smaller and often exhibit modality-specific artifacts (e.g., low-dose noise, motion artifacts, and low spatial resolution). These factors limit the overall image quality and pose challenges for designing and training high-capacity deep models. To better match, as far as possible, the training scale of DINOv2 on natural images, we aggregate MRI/CT scans from multiple public datasets. On the other hand, medical images typically contain more structured and recurring anatomical patterns (e.g., organs and bones) than natural images, which may reduce the amount of data required to learn useful representations; quantifying this effect is an interesting direction for future work. Below, we introduce each dataset individually.
\begin{figure*}[t!] 
    \centering
    \includegraphics[width=1.\linewidth]{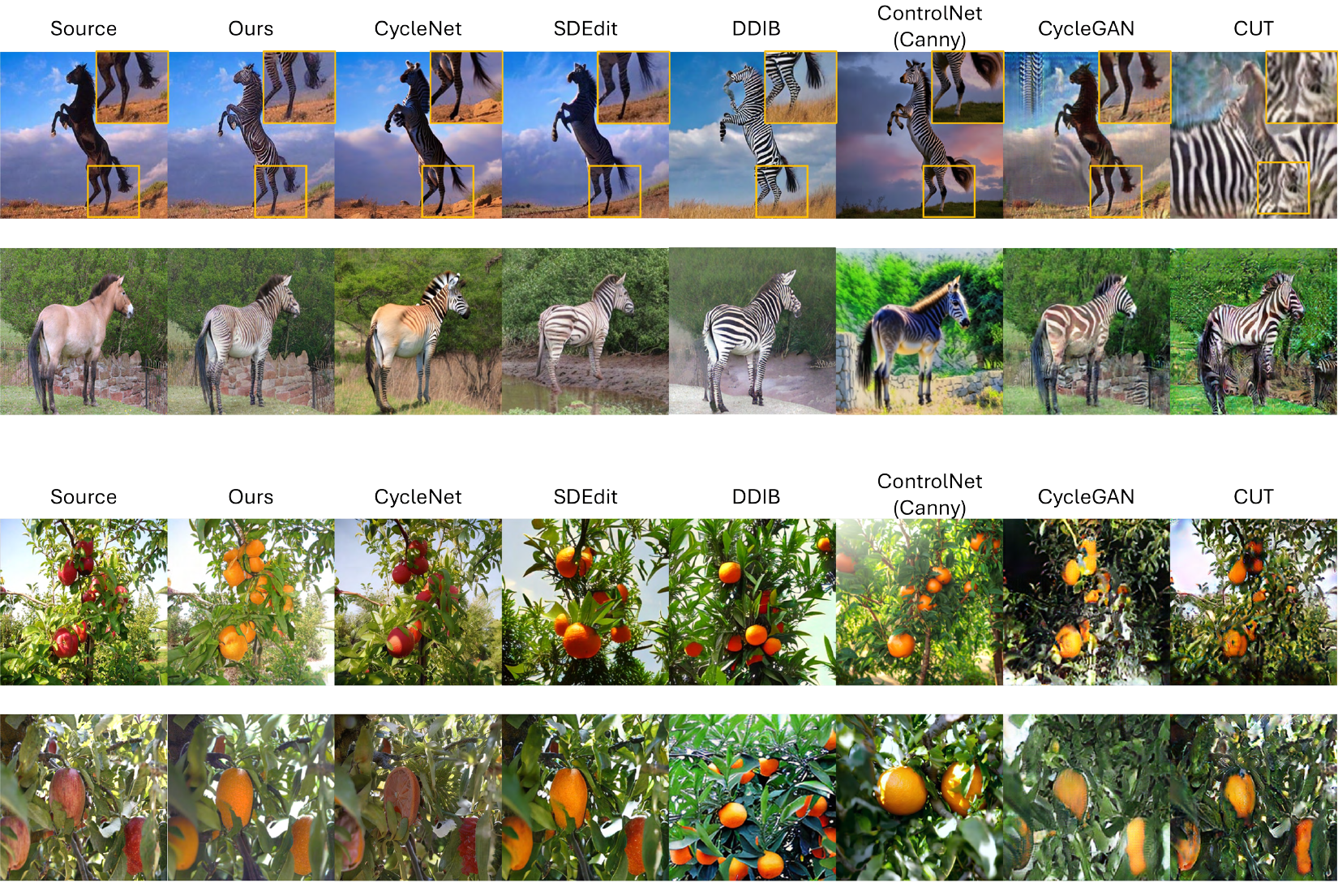}
    \vspace{-0.5em}
    \caption{\textbf{Visual Results.} Qualitative results on natural image translation. We show unpaired horse→zebra and apple→orange examples comparing SSB to prior GAN- and diffusion-based baselines. Our method achieves strong appearance transfer (striping / color change) while better preserving object pose, background layout, and fine structural details. The SSB results are obtained by fine-tuning SiT-XL/2.}
\vspace{-1em}    
    \label{fig:main_results_horsezebra}
\end{figure*}

\noindent\textbf{SynthRAD2023.} 
The SynthRAD2023~\cite{thummerer2023synthrad2023} radiotherapy dataset provides paired multimodal images (MRI–CT and CBCT–CT) for brain and pelvic treatment sites from three Dutch university medical centers. Various scanner models and acquisition settings were used across patients at these centers. In total, we selected a subset of 350 MRI–CT patient cases that are relevant to our study. The subset consists of approximately two-thirds T1-weighted spoiled gradient-echo sequences and one-third T2-weighted fast spin-echo sequences, with no contrast used for either MRI or CT. We select about 20 subjects for evaluation and treat the rest as training  dataset.

\noindent
\textbf{SynthRAD2025.}
Similarly, SynthRAD2025~\cite{thummerer2025synthrad2025} benchmarks synthetic CT generation for MRI- and CBCT-based radiotherapy workflows. The dataset contains MRI-to-CT pairs for MR-only and MR-guided photon/proton radiotherapy (890 MRI–CT pairs) and CBCT-to-CT pairs for daily adaptive workflows (1,472 CBCT–CT pairs) from patients who underwent radiotherapy in the head-and-neck, thorax, or abdomen. The population is predominantly adult, with no gender restrictions applied during data collection, and the images are aggregated from five international centers. In this work, we focus on the MRI-CT subset and use approximately 637 publicly available MRI–CT cases that match our training experimental design and 30 for testing.

\noindent\textbf{Multi-organ Abdominal Collection (AMOS).} AMOS~\cite{ji2022amos} is a multi-modal dataset from Longgang District People's Hospital, featuring 500 CT and 100 MRI scans from 600 patients with abdominal abnormalities. Acquired across eight different scanner platforms, the dataset provides annotations for 15 anatomical structures in CT (spleen, right kidney, left kidney, gallbladder, esophagus, liver, stomach, aorta, inferior vena cava, pancreas, right adrenal gland, left adrenal gland, duodenum, bladder, and prostate/uterus) and 13 structures in MRI (excluding bladder and prostate in some cases). We use both modalities for training dataset construction.

\noindent\textbf{Whole-body PET/CT Collection (AutoPET).} AutoPET~\cite{gatidis2022whole} comprises 1,014 whole-body FDG-PET/CT studies, balanced between 501 cases with confirmed malignancies (lymphoma, melanoma, NSCLC) and 513 negative control cases. All scans include both PET and CT modalities with annotations for malignant lesions. We use the CT images for our training dataset.

\noindent\textbf{TotalSegmentator.} TotalSegmentator~\cite{wasserthal2023totalsegmentator} provides comprehensive whole-body segmentation datasets. The CT portion comprises 1,229 scans with annotations for 104 anatomical structures spanning all major body regions, organs, vessels, and skeletal structures. The MR portion includes 299 scans with annotations for 50 anatomical structures. Both datasets aggregate scans from multiple sources and institutions, providing diverse imaging protocols and patient populations. We adapt both modalities for unpaired training dataset construction.

\noindent
\textbf{IXI.}
The IXI dataset~\citep{ixi_dataset} contains nearly 600 brain MR scans from healthy subjects, acquired at three London hospitals (Philips 3T, Philips 1.5T, GE 1.5T).  For each subject, T$_1$-, T$_2$- and PD-weighted images, MRA, and diffusion-weighted images are available.  Following prior work, we chose a subset of IXI, containing about 300 subjects with T$_1$/T$_2$/PD images. We exclude boundary slices that contain no brain tissue from each volume. 

\noindent
\textbf{Pelvic MRI/CT Dataset.}
This Pelvic MRI/CT~\cite{nyholm2018mr} dataset contains pelvic T$_1$-, T$_2$-weighted MRI and CT scans from 15 subjects. T$_1$ images were acquired with TE $\approx$ 4.8--7.2\,ms and TR $\approx$ 500--746\,ms at either $0.88\times0.88\times3$\,mm$^{3}$ or $1.10\times1.10\times2$\,mm$^{3}$ resolution; 
T$_2$ images with TE $\approx$ 91--97\,ms, TR $\approx$ 6000--16000\,ms and in-plane resolution 0.88--1.10\,mm at 2.5\,mm slice thickness. 
CT scans have $0.10\times0.10\times2$--3\,mm$^{3}$ resolution with reconstruction kernels B30f or FC17. 

\noindent
\textbf{UKBB.} UK Biobank’s imaging enhancement~\cite{littlejohns2020uk} is a large population study, each undergoing brain, cardiac, and abdominal/whole-body MRI, plus DXA and carotid ultrasound. The abdominal and whole-body MRI are acquired on a 1.5T system using multi-station Dixon sequences and additional liver/pancreas protocols. These body scans provide volumetric measures of visceral and subcutaneous fat, muscle volume, and organ morphology, and are explicitly designed as a gold standard for body composition and ectopic fat quantification. Since the UKBB whole-body MRI dataset exhibits different appearance and resolution compared to the in-domain SynthRAD2023 and SynthRAD2025 data, we treat it as an out-of-domain MRI dataset. We select a small subset of UKBB whole-body MRI with both fat-only and water-only Dixon reconstructions, containing about 100 subjects each. For each volume, we extract 5 axial slices spanning from the mid-thigh to the chest.  

\noindent\textbf{In-domain vs.\ Out-of-domain (Empirical Justification).}
Although all models in our comparison are trained in an unpaired fashion, the underlying MRI and CT datasets differ substantially in their acquisition compatibility. The SynthRAD2023/2025 MRI and CT volumes were originally collected for supervised MRI$\rightarrow$CT research and therefore share similar spatial resolution, anatomical framing, and pre-registered geometry. As a result, the marginal distributions of SynthRAD MRI and CT are naturally closer. Empirically, the Fréchet Inception Distance (FID) between SynthRAD MRI and SynthRAD CT is 120.7. In contrast, UKBB fat- and water-suppressed MRI have different contrast mechanisms, intensity statistics, and voxel sizes, and are not geometrically aligned to any CT dataset we use. Consequently, their distributional distance to SynthRAD CT is substantially larger (FID of 177.2 and 191.6, respectively). Even UKBB fat- and water-suppressed MRIs differ from each other (FID 69.6), highlighting their contrast heterogeneity.
\begin{table}[t]
\caption{\textbf{FID distances between MRI and CT datasets.} 
U-FAT = UKBB fat-suppressed MRI; U-WATER = UKBB water-suppressed MRI; 
S-MRI = SynthRAD2023/2025 MRI; S-CT = SynthRAD2023/2025 CT.
SynthRAD MRI is substantially closer to SynthRAD CT (FID 120.7) than UKBB MRI (FID 177–192), supporting our definition of SynthRAD MRI as “in-domain’’ and UKBB MRI as “OOD’’ with respect to CT acquisition.}
\centering
\small
\resizebox{0.42\textwidth}{!}{
\begin{tabular}{lcccc}
\toprule
\textbf{FID Dist.$\downarrow$} & \textbf{U-FAT} & \textbf{U-WATER} & \textbf{S-MRI} & \textbf{S-CT} \\
\midrule
\textbf{U-FAT}   & 0     & 73.5  & 97.6  & 177.2 \\
\textbf{U-WATER} & ---   & 0     & 110.4 & 191.6 \\
\textbf{S-MRI}   & ---   & ---   & 0     & 120.7 \\
\textbf{S-CT}    & ---   & ---   & ---   & 0     \\
\bottomrule
\end{tabular}
}
\vspace{-1em}
\label{tab:fid_mri_ct}
\end{table}

We therefore use “in-domain MRI’’ to denote MRI acquisitions whose marginal distribution is compatible with the CT dataset used in evaluation (as in SynthRAD), and “out-of-domain MRI’’ to denote acquisitions that differ strongly in resolution and contrast (as in UKBB). This distinction reflects data compatibility, not supervision. On the other hand, this combination of in-domain and out-of-domain datasets can also be viewed as a form of multi-contrast MRI translation, where diverse MRI contrasts (e.g., T1, T2, water–fat images) are translated into a unified CT domain.

\noindent
\subsection{Natural Images} 

In this section, we introduce the datasets used to finetune the SiT class-label ImageNet flow transformers and the text–image SD3-M transformers employed in our main paper. For SiT, we follow~\citep{yu2025representation} and first resize ImageNet-1K~\cite{imagenet15russakovsky} images to $256\times256$ and encode them into a VAE latent space of spatial size $32\times32$ with 4 channels. ImageNet-1K spans 1000 object classes and contains 1{,}281{,}167 training images.

For SD3-M finetuning, we adopt LAION-POP and a high-resolution, non-overlapping subset of LAION-2B, both derived from LAION-5B~\cite{schuhmann2022laion}. LAION-POP contains approximately 600{,}000 high-resolution images emphasizing popular text-to-image concepts and is heavily filtered for high aesthetic scores and a minimum resolution of 768\,px on the shorter side, making it suitable for finetuning text–image models. We then sample an additional subset from LAION-2B by requiring the shorter side to exceed 800\,px and enforcing non-overlap with LAION-POP. Together, we choose about 1.2M high-quality text–image pairs in total.
\begin{algorithm}[htb] 
\caption{Simplified SSB for Conditional I2I Translation and Editing}
\label{alg:mixed-rev-ode}
\begin{algorithmic}[1]
\REQUIRE 
Source $\vx^{(j)}_0$; encoder $E_\phi$; decoder $D_\varphi$; inverter $\vz^{\mathrm{inv}}_T$; drifts $v^{(i)}_\theta, {\color{gray}v^{(j)}_\theta}$; grid $\{t_k\}_0^N$; $p_w$; $\eta_t$.
\ENSURE Translated latent $\bar\vz^{(i)}_0$; .

\STATE $\vy \leftarrow E_\phi(\vx^{(j)}_0)$ \textbf{or} {\color{gray}$\vz^{\mathrm{inv}}_T(\vx^{(j)}_0)$}
\STATE $\vz_T \leftarrow \vy$
\STATE $\vz^{(i)}_{t_0}, {\color{gray}\vz^{(j)}_{t_0}, \tilde\vz^{(i)}_{t_0}} \leftarrow \vz_T$

\FOR{$k=0$ \TO $N-1$}
    \STATE $t \leftarrow t_k,\;\Delta t \leftarrow t_{k+1}-t_k$
    
    \STATE $\texttt{cfg\_on} \leftarrow (s>1)\wedge(t\in[t_{\min},t_{\max}])$
    \IF{$\texttt{cfg\_on}$}
        \STATE $d_t^{(i)} \leftarrow v^{(i)}_{\text{uncond}} + s\,(v^{(i)}_{\text{cond}} - v^{(i)}_{\text{uncond}})$
    \ELSE
        \STATE $d_t^{(i)} \leftarrow v^{(i)}(t, \vz_t^{(i)}, \vz_T)$ 
    \ENDIF

    \STATE {\color{gray} $d_t^{(j)} \leftarrow v^{(j)}(t, \vz_t^{(j)}, \vz_T)$ \hfill (Optional)}
    \STATE {\color{gray} $d\tilde{\vz}_t^{(i)} \leftarrow d_t^{(i)} + \eta_t(d_t^{(j)} - d_t^{(i)})$ \hfill (Optional)}
    
    \STATE $\vz_{t_{k+1}}^{(i)} \leftarrow \vz_{t_k}^{(i)} + \Delta t \cdot d^{(i)}_t$
    \STATE {\color{gray} $\{\vz^{(j)}, \tilde{\vz}^{(i)}\}_{t_{k+1}} \leftarrow \{\vz^{(j)}, \tilde{\vz}^{(i)}\}_{t_k} + \Delta t \cdot \{d^{(j)}, d\tilde{\vz}^{(i)}\}_t$ \hfill (Optional)}
\ENDFOR

\RETURN~~$\bar{\vx}^{(i)}_0 \leftarrow D_\varphi(\tilde{\vz}^{(i)}_{t_N})$ (or $D_\varphi(\vz^{(i)}_{t_N})$ if source model ${\color{gray}v^{(j)}_\theta}$ is not used)
\end{algorithmic}
\end{algorithm}

\section{Additional Implementation Details}
Below we provide additional implementation details for our method and the baseline models that are obmitted from the main paper due to space limitations.~\emph{It is worth noting that in our work we do not explicitly tune the endpoint Gaussian variance parameter $b$, which appears in $\vz_T=\vy\sim \gN(E_\phi(\vx_0), b^2\mathbf{I})$}. We include the notation for $b$ primarily to~\emph{highlight the conceptual distinction} between different bridge constructions discussed in the main paper—namely, the medical image translation setting, where geometric alignment is strong, versus the natural image generation setting, which involves richer semantic ambiguity. All baseline methods and our own models are trained on a cluster equipped with AMD MI300 GPUs.

\subsection{DINOv2 Pre-training for MRI-CT Translation}
\label{Sec:dinov2_medical}
In this section, we briefly describe the MRI–CT augmentation pipeline used to train the DINOv2 ViT-B/8 encoder (Fig.~\ref{fig:augmentation-pipline}). As noted in the Remark in the main paper, our objective is to construct an augmentation strategy such that, at convergence, the vision encoder becomes largely invariant to appearance variations while retaining stable geometric structure. To achieve this, apart from commonly used augmentations, we incorporate a retinal-inspired contrast filter implementing a spatial–temporal Gaussian center–surround mechanism (a weighted Difference-of-Gaussians). This filter suppresses low-frequency appearance cues and enhances high-frequency structural information, providing geometry-preserving augmentations that are suitable for contrast-agnostic MRI–CT representation learning. The filtering effect on paired MRI-CT is introduced in Fig.~2 (\emph{bottom}). The appearance gap between CT-MRI is largely reduced by run few iterations of the filter. Following the original DINOv2 training settings, we set the batch size to 512. The final model is obtained in approximately one day, requiring over 100,000 iterations.

\noindent
\textbf{Retina-Inspired Filter.}
Early retinal circuitry is well-approximated by a center--surround receptive field, in which a narrow positive ``center'' is opposed by a broader negative ``surround.'' This structure functions as a spatial band-pass operator that enhances local contrast and boundaries while suppressing slowly varying background intensity. Formally, given an image $f(x,y)$, we define a retinal-inspired response
\begin{equation}
R(x,y;t)
=
\big(G_{\sigma_c} * f\big)(x,y)
\;-\;
w_s\,\big(G_{\sigma_s(t)} * f\big)(x,y),
\label{eq:retina}
\end{equation}
where $G_\sigma$ is an isotropic Gaussian kernel of standard deviation $\sigma$, $\sigma_c$ is the fixed center scale, and $\sigma_s(t) > \sigma_c$ is a surround scale that increases with a time-like parameter $t$. Larger~$t$ produces a broader surround and thus a stronger center--surround effect. 
\begin{figure}[t] 
    \centering
    \includegraphics[width=0.75\linewidth]{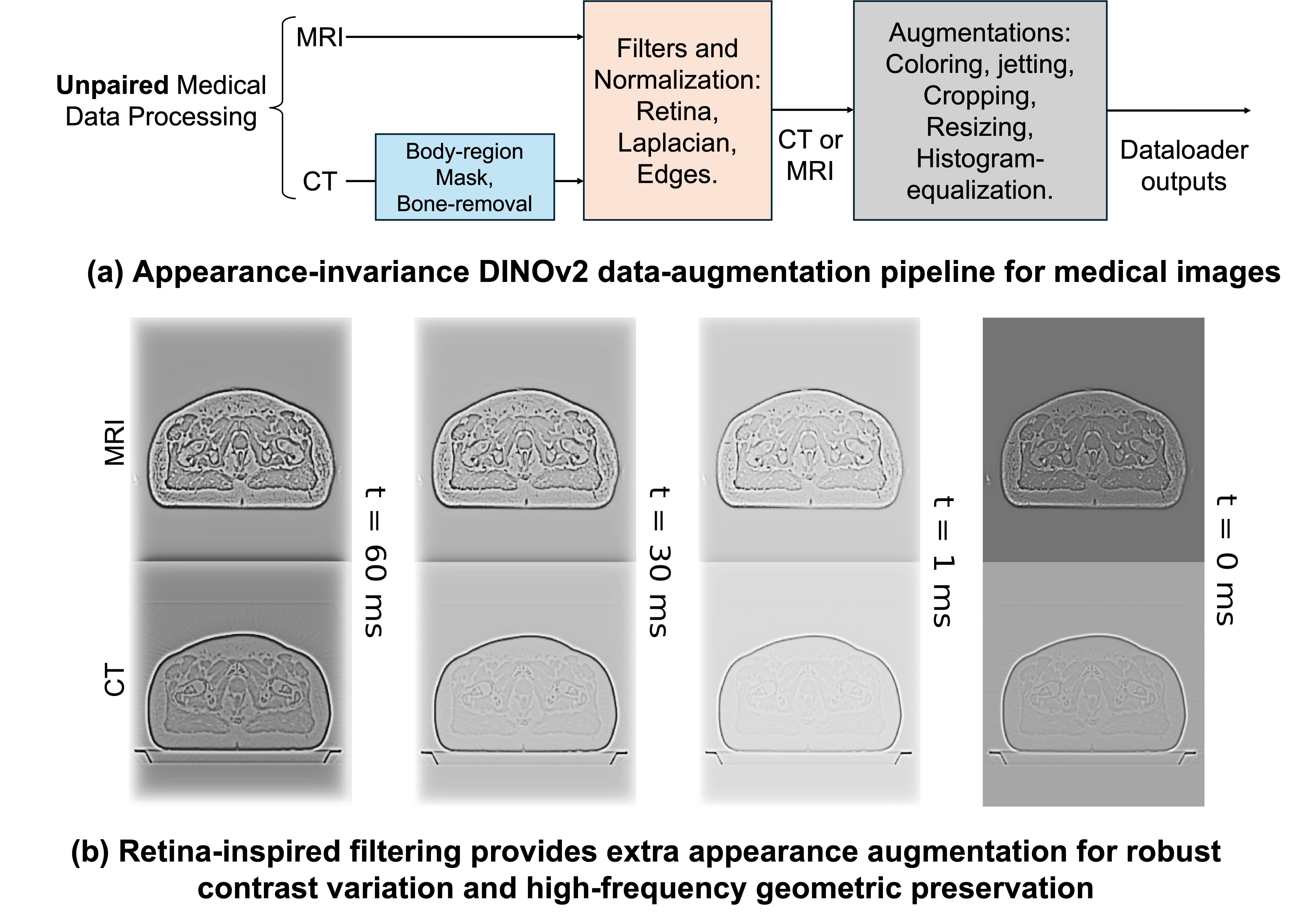}
    \caption{Data augmentation pipeline used for MRI-CT DINOv2
vision encoder. We adapt the widely used spatial-temporal retinainspired filter for aligning MRI-CT appearance gap as illustrated in the bottom of this figure.}
    \vspace{-1em}
    \label{fig:augmentation-pipline}
\end{figure}

\paragraph{Practical Note.}
In our preprocessing, we use \texttt{cv::bioinspired::Retina}. The parvocellular (detail) output of this model behaves consistently with the formulation above and can be interpreted as snapshots of $(\varphi_t * f)$ evaluated at increasing effective times~$t$. In this implementation, running the same input frame for $m$ internal iterations corresponds to evaluating the WDoG at a later time index, as each iteration expands the surround via a diffusion-like update. Typical retinal settings follow $\sigma_s \approx 3\sigma_c$ with positive center and surround gains $(w_c, w_s)$; we adopt the library’s default configuration.

\subsection{Medical I2I Translation Task---Bridge Model Training for CT Images }
\label{Sec:bm_medical}

In this section, we provide additional details on training the diffusion bridge for MRI→CT translation. MRI and CT exhibit strong geometric consistency and low semantic ambiguity, so the vision encoder $E_\phi$ already produces a shared latent space. This allows us to train a \emph{single} bridge model with $b=0$, yielding $\vz_T=\vy=E_\phi(\vx_0)$. We use a KL-VAE~\cite{rombach2022high} with $64\times64$ spatial resolution and 8 latent channels, trained on the same combined MRI/CT dataset. In practice, since all exiting medical SOTA baselines are implemented based on UNet diffusion model~\cite{ho2020denoising}, so we adapt the same type of UNet architecture. Our final bridge operates on $64\times64\times 8$ latents. Following standard bridge constructions~\cite{zhou2024denoising, pmlr-v202-liu23ai, chadebec2025lbm}, we concatenate endpoints as UNet conditioning and apply a top-16 PCA channel sampling (channel-wise shuffle) at each iteration to encourage appearance-invariant learning and improve robustness across MRI contrasts. We set the training batch size to 128 and Adam optimizer with base lr=$1e-5$, and the final model used in this work is obtained in approximately one day.

\subsection{Text-free I2I Translation Task---SiT Transformers Finetuning.}
\label{Sec:sit_sec}
In this experiments, we further study SSB in a lower-capacity setting by fine-tuning it on class-conditional flow transformers. Specifically, we use the pretrained 256$\times$256 ImageNet-1K SiT-XL/2 model~\cite{ma2024sit, yu2024representation}, trained as a rectified flow with a linear interpolant schedule and a $32\times 32$ KL-VAE latent grid with 4 channels. To match the Gaussian prior of the rectified-flow formulation, we set $b=1$ during SSB adaptation, resulting in the endpoint distribution $\vz_T = \vy \sim \mathcal N(E_\phi(\vx_0),\,\mathbf I)$. To align with SiT’s 4-channel latent space, we construct $\vz_T$ by average-pooling the top 16 PCA components of the DINOv3 embedding into a 4-channel tensor of size $32\times32\times4$. These pooled features serve as the terminal endpoint for the diffusion bridge. 

Following the stand bridge design~\cite{zhou2024denoising, zhang2024exploring, chadebec2025lbm}, we incorporate feature injection, allowing the drift network $v_\theta$ to receive auxiliary feature inputs from the source encoding $E(\vx_0^{(j)})$. This conditioning is trained in a classifier-free guidance manner~\cite{ho2021classifierfree}, allowing controllable translation strength. In specific, we inject PCA-based structural conditioning through a zero-initialized PatchEmbed module
$\mathrm{cond_embed}$, added to the SiT input tokens with a learnable scale $c$

\begin{equation}
x = \mathrm{x\_embed}(x) + \mathrm{pos}
+ c\,\underbrace{\mathrm{cond\_embed}(w \odot \mathrm{cond})}_{\text{zero-init}},
\label{eq:x-embedding}
\end{equation}
with $ w \sim \mathrm{Bernoulli}(p)$. Zero-init ensures a smooth transition from the pretrained SiT, while stochastic conditioning dropout
(probability $p_w=0.2$) prevents over-reliance on the PCA signal. This yields more stable training and produces a
flexible, geometry-aware conditioning mechanism for structure-preserving edits. Similar to the original SiT-XL/2 training protocols, we set the batch size to 512 using the AdamW optimizer with a base learning rate of $5 \times 10^{-7}$. The fine-tuning process proves highly efficient, yielding the final model in approximately 4 hours.
\begin{figure*}[t!] 
    \centering
    \includegraphics[width=0.98\linewidth]{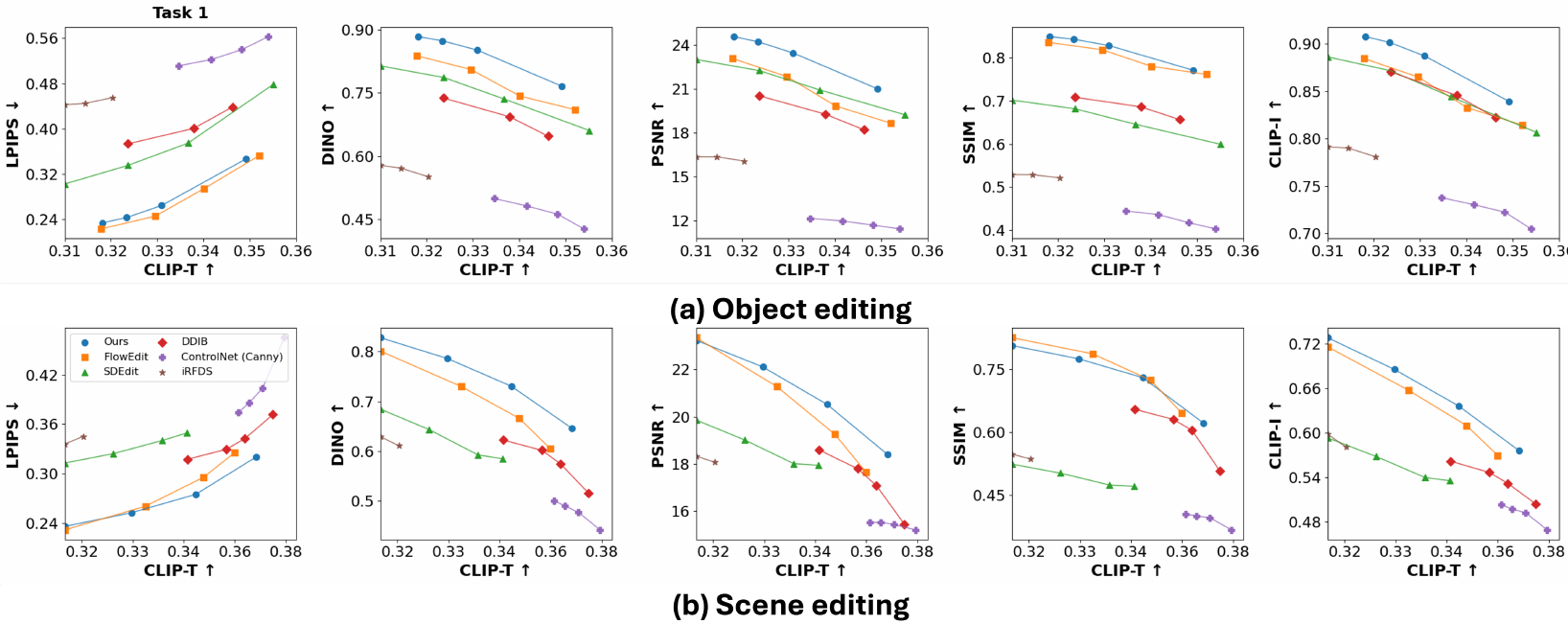}
    \vspace{-1em}
    \caption{Full trade-off curves complementing the quantitative results shown in the main paper Fig~6. As in the main text, we plot text alignment (CLIP-T, x-axis) against multiple fidelity and perceptual metrics (y-axis), including LPIPS, DINO, PSNR, SSIM, and CLIP-I. These expanded plots show the complete range of guidance settings used for each method. Consistent with the main-paper results, our method provides a favorable balance between text adherence and structural preservation, comparing competitively with—and in many cases outperforming—state-of-the-art baselines such as FlowEdit. In addition, all methods exhibit smooth and monotonic behavior as guidance strength varies, confirming the stability of our evaluation protocol. ControlNet (Canny) performs substantially worse due to its edge-only conditioning, which discards most semantic and photometric information. Structural fidelity for~\emph{scene-style} editing metrics (SSIM, PSNR, LPIPS) are computed on the luminance (YCbCr) channel to better reflect perceptual structure consistency. Connected markers correspond to different hyperparameter values for each baseline.}
    \label{fig:table2_appendix}
\end{figure*}

\subsection{Text-Guided Image Editing---SD3-M Transformer Finetuning}
\label{Sec:sd3_finetune}
We fine-tune SSB on the SD3-M text–image flow transformer~\cite{esser2024scaling}, which employs a 16-channel KL-VAE and preserves the original $1024\times1024$ resolution (yielding a $128\times128\times16$ latent). Following the rectified-flow setup, we set $b=1$ so the endpoint distribution aligns with the Gaussian prior, giving $\vz_T=\vy\sim\mathcal N(E_\phi(\vx_0),\mathbf I)$. Unlike the SiT setting, directly using the full top-16 DINOv3 PCA channels as endpoints severely degrades finetuning stability on SD3-M. We attribute this to the model’s more complex flow trajectories and the larger mismatch between its 16-channel VAE latent space and the DINOv3 embedding space. As a result, endpoints formed solely from DINO PCA features lie outside the native SD3-M latent manifold and destabilize SSB training. To address this mismatch, we construct a hybrid endpoint using channel-wise mixing between the two latent
encoders. Specifically, we sample a binary channel mask $m\in\{0,1\}^C$ and define
\[
E'_\phi(\vx_0)
= m \odot E_\varphi(\vx_0)
+ (1-m)\odot E_\phi(\vx_0),
\]
where $E_\varphi$ is the SD3-M VAE encoder and $E_\phi$ denotes the DINOv3 PCA features. This mechanism randomly replaces a subset of DINO PCA channels with the corresponding VAE latent channels, thereby anchoring the endpoint closer to the native SD3-M latent space while preserving the geometry-aware structure captured by DINO. We find this hybrid construction essential for stable and high-quality SSB adaptation on SD3-M. See additional discussions and experimental results in the next Section. Similar to SiT finetuning, in SD3-M, we also inject conditioning into several transformer blocks using standard zero-initialized residual adapters. This allows the SSB signal to influence deeper layers while preserving pretrained behavior at initialization. For SD3-M fine-tuning, we use a batch size of 256 and the AdamW optimizer with a learning rate of $1.2 \times 10^{-4}$. Similar to the SiT fine-tuning stage, convergence is rapid, requiring approximately 8 hours to obtain the final model.

\subsection{Parameter Settings of SSB at Inference }
In Table~\ref{tab:parameters_ours}, we summarize the hyperparameter settings used for MRI$\rightarrow$CT, horse$\rightarrow$zebra, text-guided scene editing, and object-editing experiments. As noted earlier, the MRI$\rightarrow$CT task does not require an inversion step, since our pre-trained DINOv2 encoder already provides a strong geometry-preserving representation of the source MRI, making direct conditioning sufficient.
For scene-level image editing, we prioritize more dramatic global appearance changes. In this setting, model inversion offers limited benefit because the structural prior encoded by $E_\phi$ already retains rich spatial information while largely disentangling appearance. This enables effective scene translation without relying heavily on inversion, leading to diverse scene level translation, whereas object-level editing—where fine-grained control is needed—may benefit more from it. Following the hyperparameter configurations detailed in Table~\ref{tab:parameters_ours}, we present our practical reverse ODE sampler based on Euler discretizations in Algorithm~\ref{alg:mixed-rev-ode}. 

\subsection{Baseline Methods Implementations}
\label{Sec:baseline_implt}
In this section, we briefly introduce each baseline methods used in the experiments and their implementation details. 
\begin{table*}[t]
\caption{Control strengths for inversion/editing baselines (unified across each method’s original notation). For SDEdit, FlowEdit, and DDIB, a larger strength corresponds to weaker preservation of the source structure, whereas for ControlNet the relationship is reversed.}
\centering
\resizebox{0.9\textwidth}{!}{
\begin{tabular}{lccccc}
\toprule
\textbf{Methods} &
\textbf{$T$ steps} &
\textbf{$t_{\text{end}}$ (Control Strength*)} &
\textbf{$n_{\text{avg}}$} &
\textbf{CFG @ source} &
\textbf{CFG @ target} \\
\midrule
SDEdit~\cite{meng2022sdedit} & 75 & 
0.2,\; 0.3,\; 0.4,\; 0.5,\; 0.6,\; 0.75
& 1
& -- & 7.5,\; 13.5,\; 16.5 \\
DDIB~\cite{su2022dual} & 200 & 0.7, 0.8, 0.9 & 1 & 1,\; 3.5 & 13.5,\; 16.5,\; 19.5 \\
FlowEdit~\cite{kulikov2024flowedit} & 50 & 
0.5,\; 0.56,\; 0.6,\; 0.66,\; 0.7,\; 0.8
& 1, 3, 5
& 1,\; 3.5 & 13.5,\; 16.5,\; 19.5 \\
ControlNet~\cite{zhang2023adding} & 75 & 
0.2,\; 0.5,\; 0.7,\; 0.8,\; 0.9 
& 1
& -- & 7.5,\; 13.5,\; 16.5 \\
iRFDS~\cite{yang2025texttoimage} & 1000, 1400, 1800 & 
--
& 1
& -- & official hyperparameters \\
\bottomrule
\end{tabular}
}
\label{tab:baseline_parameters}
\end{table*}

\begin{table*}[t]
\caption{
SSB hyperparameter configurations across different I2I translation tasks (medical and natural images). 
We utilize \textbf{standard diffusion hyperparameters} (e.g., inversion steps, control strength) consistent with prior I2I methods, adapting settings like stochasticity ($\bm{p_w}$) and guidance strength solely based on the task's intrinsic ambiguity. 
This demonstrates that SSB is a general framework achieving strong performance across domains using established configurations, without relying on complex, task-specific heuristics.
}
\centering
\resizebox{0.9\textwidth}{!}{
\begin{tabular}{llccccc}
\toprule
\textbf{Task} & \textbf{Setting} & \textbf{Inversion Step} & $\bm{T}$ \textbf{steps} & $\bm{t_{\text{end}}}$ (\textbf{Control Strength*}) & $\bm{p_{w}}$ \textbf{in Eq.}~\eqref{eq:x-embedding} & \textbf{CFG @ target} \\
\midrule

\multirow{1}{*}{\textbf{Medical I2I}} 
& MRI$\rightarrow$CT & -- & 50 & -- & -- & -- \\
\midrule

\multirow{2}{*}{\textbf{Text-free I2I}}
& Horse$\rightarrow$Zebra  & 100 & 50 & 0.2, 0.4, 0.7 & 0.5, 0.7 & 3, 4.5 \\
& Apple$\rightarrow$Orange & 100 & 50 & 0.2, 0.4, 0.7 & 0.5, 0.7  & 3, 4.5 \\
\midrule

\multirow{2}{*}{\textbf{Text-guided I2I}}
& Scene editing  & -- & 75 & 0.4, 0.7, 1 & 0, 0.2, 0.5, 0.7 & 4.5, 7.5, 13.5 \\
& Object editing & 75 & 75 & 0, 0.2, 0.4, 0.7 & 0, 0.2, 0.5, 0.7 & 4.5, 7.5, 13.5 \\
\bottomrule
\end{tabular}
}
\label{tab:parameters_ours}
\end{table*}

\noindent
\textbf{GAN-based Methods.}
We include classical GAN-based I2I baselines—CycleGAN~\citep{zhu2017unpaired}, UNIT~\citep{liu2017unsupervised}, and CUT~\citep{park2020contrastive}—to provide a complete comparison. CycleGAN and UNIT are trained on our MRI→CT data, while CUT is evaluated on the official horse→zebra model and fine-tuned for apple→orange. Because these architectures have much smaller capacity than modern diffusion or flow-based models, we construct a fair unpaired MRI/CT setting by sampling 5k MRI and 5k CT images from SynthRAD2023/2025 and another 5k from the remaining data, ensuring comparable training conditions without domain coupling.

\noindent
\textbf{Supervised Diffusion Bridge-based Methods.}
Since the MRI-CT dataset used in this paper is originally released for supervised training design, with pre-registered MRI-CT pairs, we compare our methods with

\noindent
\textbf{Cycle-consistency Diffusion-based Methods.}
We compare against SynDiff~\citep{ozbey2023unsupervised}, a recent cycle-consistent diffusion–GAN method tailored for MRI→CT translation. SynDiff is designed for pre-registered MRI/CT pairs (e.g., SynthRAD2023/2025), where MRI and CT share similar resolution and alignment, and it trains a separate model for each contrast (\emph{e.g.}, T1→CT, T2→CT). However, SynDiff is highly sensitive to domain coupling and does not train reliably on large-scale, heterogeneous, or unregistered data such as UKBB MRI→CT. In practice, it collapses on our full dataset, making a direct large-scale comparison infeasible and uninformative. To provide a fair comparison under its intended operating regime, we therefore construct a smaller, higher-quality unpaired subset by combining SynthRAD2023/2025 with an additional 10,000 unpaired MRI/CT images from the remaining training pool. 

CycleNet~\cite{xu2023cyclenet} has demonstrated strong performance on classical unpaired I2I benchmarks such as horse→zebra and apple→orange by fine-tuning pretrained text-to-image diffusion models (e.g., Stable Diffusion 2.1~\cite{rombach2022high}). For comparison, we use the officially released CycleNet model pretrained on horse→zebra and further fine-tune it on the apple→orange dataset~\cite{zhu2017unpaired} following its recommended protocol. This provides a comprehensive baseline representing the line of diffusion-based cycle-consistency methods. We follow the official implementations of the parameter settings for CycleNet.

\noindent
\textbf{Inversion-based Flow Methods.}
For natural image I2I, we evaluate inversion-based baselines including DDIB~\cite{su2022dual} and SDEdit~\cite{meng2022sdedit}. Both methods rely on a pretrained flow model and do not learn cross-domain mappings; accordingly, we use the same official SiT model as their backbone, identical to ours. For text-guided I2I editing, we use the official SD3-M model as the pretrained backbone for fair comparison. We additionally compare with FlowEdit~\cite{kulikov2024flowedit}, a recent SOTA, inversion-free editing method specifically designed for rectified-flow models, providing a strong baseline for text-guided structural transformations. In Table~\ref{tab:baseline_parameters}, we show the hyperparameter settings used in this work for each above baselines.

\noindent
\textbf{ControlNet.} For all exaperiments used related to ControlNet, we adopt the pre-trained released in SD3-Controlnet-Canny~\citep{instantx_sd3_controlnet_canny} on huggingface. 

\subsection{Kernel Alignment Matrices}
\label{Sec:CKNNA}
\noindent
\textbf{CKNNA.}
Centered Kernel Nearest-Neighbor Alignment (CKNNA) is a relaxed variant of the standard Centered Kernel Alignment (CKA)~\citep{kornblith2019similarity}, designed to soften CKA’s strict global alignment criterion by focusing only on locally consistent neighborhoods. We adopt the notation and formulation introduced in Huh et al.~\citep{huh2024position}. We follow~\cite{yu2024representation} for the metrics implementation in Fig.3 in the main paper. In this paper, we randomly sample 1,000 paired MRI–CT slices from the training set and report CKNNA@$k=10$ computed on this subset and averaged across all layers of each image encoder; following Huh et al.~\cite{huh2024position}, smaller $k$ provides a more stringent assessment of local neighborhood consistency between representations.

\section{Additional Details to Related Work}
\noindent\textbf{Connecting Conditional SI, ECSI, and DDBM.} Beyond the vector field interpolation, one can also formulate conditional sampling from the SDE perspective in practice. As is well known~\citep{ho2020denoising, song2021scorebased, meng2022sdedit, rout2025semantic}, SDEs exhibit robustness to initialization, with stability proportional to the variance of the additive noise. In contrast, ODE trajectories initialized with corrupted or imperfect samples propagate these errors deterministically through the velocity field. By injecting stochasticity at every step, SDEs behave as Markov processes that converge toward the intended invariant distribution, thereby reducing sensitivity to initialization and improving the reliability of conditional generation.

\noindent
We define the reverse SDE that shares the conditional marginals $p(\vz_t \mid \vz_T)$ with the PF-ODE as
\begin{equation}
d\vz_t = \big[v(t,\vz_t,\vz_T) - g_ts(t,\vz_t,\vz_T)\big]dt + \sqrt{2g_t}d\vw_t,
\label{eq:rev-sde-generic}
\end{equation}
where $g_t>0$ is the diffusion coefficient, $s(t,\vz_t,\vz_T)=\nabla_{\vz}\log p(\vz_t \mid \vz_T)$ is the conditional score, and $\vw_t$ is a standard Wiener process. Setting $g_t \equiv  0$ recovers the deterministic PF-ODE. Conditioned on the endpoint and using Tweedie’s identity, the score is
\begin{equation}
s(t, \vz, z_T) := \nabla_\vz \log p(\vz \mid \vz_T)
    = -\tfrac{1}{\gamma_t}\,\mathbb E[\beps \mid \vz_t=\vz, \vz_T].
\end{equation}
By combining this score identity with the stochastic interpolant parameterization, we obtain a closed-form correspondence between score and velocity as 
\begin{align}
v(t,\vz,\vz_T)
&= \E\!\left[\,\partial_t I(t,\vz_0,\vz_T)\mid \vz_t=\vz,\,\vz_T\,\right] \notag \\
&\quad -\; \dot\gamma_t \gamma_t \, s(t,\vz,\vz_T).
\label{eq:score-vector}
\end{align}
We now show that the velocity--score reverse SDE for 
SI~\cite{albergo2023stochastic} coincides with the 
ESCI~\cite{zhang2024exploring} formulation, which in turn 
is equivalent to diffusion bridge models (DDBM)~\cite{zhou2024denoising} under linear gaussian defined in the main paper and restated below.

\begin{table}[t]
    \centering
    \caption{\textbf{Efficiency Comparison} on medical MRI$\rightarrow$CT $256\times256$ images. All methods use a \textbf{UNet} backbone. We evaluate SDEdit, DDIB, and DDBM using our own pre-trained diffusion models within their official frameworks. For Syndiff, SelfDRB, and I$^{2}$SB, we utilize their official implementations.}
    \label{tab:efficiency_supp}
    \resizebox{0.6\columnwidth}{!}{
    \begin{tabular}{l c c c}
    \toprule
    \multicolumn{4}{c}{\textbf{Backbone: UNet-based~\cite{ronneberger2015u, dhariwal2021diffusion}}} \\
    \midrule
    \textbf{Methods} & \textbf{NFE}$\downarrow$ & \textbf{Params} & \textbf{Inference Time (s/image)}$\downarrow$ \\
    \midrule
    SDEdit~\cite{meng2022sdedit}   & 400  & 552.8M & 8.03 \\
    DDIB~\cite{zheng2025diffusion} & 1000 & 552.8M & 20.15 \\
    Syndiff~\cite{ozbey2023unsupervised} & 4    & 156.1M & 0.13 \\
    DDBM~\cite{zhou2024denoising}  & 300  & 489.7M & 5.21 \\
    I$^{2}$SB~\cite{pmlr-v202-liu23ai} & 1000 & 552.8M & 24.97 \\
    SelfDRB~\cite{arslan2025self}  & 20   & 46.9M  & 0.29 \\
    SSB (ours)                     & 150  & 489.7M & 2.57 \\
    \bottomrule
    \end{tabular}
    }
\label{tab:mri-run}    
\end{table}

\begin{table}[t]
\centering
\caption{\textbf{Efficiency Comparison} on text-free I2I translation with \textbf{Transformer-based diffusion denoisers}. Methods differ in base architectures (e.g., SiT/SD/SD3 backbones), reflecting the recent shift from U-Net to Transformer denoisers. We use the official implementation of \href{https://github.com/sled-group/CycleNet}{CycleNet} and adapt publicly available weights for \href{https://huggingface.co/InstantX/SD3-Controlnet-Canny}{ControlNet-Canny}.}
\label{tab:efficiency_breakdown}
\resizebox{0.8\linewidth}{!}{%
\begin{tabular}{lcccc}
\toprule
\textbf{Methods} & \textbf{Base-model} & \textbf{NFE}$\downarrow$ & \textbf{Inference Time (s/image)} $\downarrow$ & \textbf{Parameters} \\
\midrule
SDEdit~\cite{meng2022sdedit}     & SiT-XL/2~\cite{yu2024representation} & 400  & 9.22  & 675M \\
DDIB~\cite{su2022dual}          & SiT-XL/2~\cite{yu2024representation} & 1000 & 22.06 & 675M \\
CycleNet~\cite{xu2023cyclenet}  & SD-2.1~\cite{rombach2022high}   & 200  & 5.21  & 865.91M \\
ControlNet~\cite{zhang2023adding} & SD3-M~\cite{esser2024scaling}  & 150  & 5.75  & 2.02B \\
SSB (Ours)                            & SiT-XL/2~\cite{yu2024representation} & 250  & 6.56  & 675M \\
\bottomrule
\end{tabular}}
\label{tab:nature-run}    
\end{table}

\begin{table}[b]
\centering
\caption{\textbf{Efficiency Comparison} on text-guided I2I translation and editing. We employ pre-trained public available SD3-M for all methods. Since the NFE varies due to hyperparameter settings (e.g., control strength) for all baselines, we omit a single NFE value and instead report the average running time across different settings to reflect the actual computational cost more accurately.}
\label{tab:runtime_simple}
\resizebox{0.5\linewidth}{!}{%
\begin{tabular}{lcc}
\toprule
\textbf{Methods} & \textbf{Inference Time (s/image)} $\downarrow$ & \textbf{Parameters} \\
\midrule
SDEdit~\cite{meng2022sdedit}     & 4.71  & \multirow{6}{*}{2.02B} \\
DDIB~\cite{su2022dual}          & 28.75 & \\
FlowEdit~\cite{kulikov2024flowedit} & 12.21 & \\
iRFDS~\cite{yang2025texttoimage}     & 44.82 & \\
ControlNet~\cite{zhang2023adding}    & 5.85  & \\
SSB (Ours)                             & 14.13 & \\
\bottomrule
\end{tabular}}
\label{tab:text-guide-run}    
\end{table}

For the linear--Gaussian SI kernel defined in the main paper, we have the form of 
\begin{equation}
\vz_t = \alpha_t \vz_0 + \beta_t \vz_T + \gamma_t \beps,
\qquad \beps \sim \mathcal N(0,I),
\end{equation}
applying Eq.~\eqref{eq:score-vector} the probability flow drift is
\begin{equation}
v(t,\vz,\vz_T)
= \dot\alpha_t \hat \vz_0 + \dot\beta_t \vz_T 
- \dot\gamma_t \gamma_t \, s_c(t,\vz,\vz_T),
\end{equation}
where we have $\hat \vz_0 = \mathbb E[\vz_0 | \vz_t, \vz_T]$ and 
$s(t,\vz,\vz_T) = \nabla_z \log \rho_t(\vz \mid \vz_T)$. By Tweedie’s identity~\cite{albergo2023stochastic} defined as
\begin{equation}
s_c(t,\vz,\vz_T) = -\tfrac{1}{\gamma_t}\,\hat \vu_t,
\qquad 
\hat \vu_t = \mathbb E[\beps \mid \vz_t,\vz_T].
\label{eq:Tweed}
\end{equation}
we directly plug-in Eq.~\eqref{eq:Tweed} to Eq.\eqref{eq:rev-sde-generic}, this lead to
\begin{equation}
d\vz_t = \Big(\dot\alpha_t \hat \vz_0 + \dot\beta_t \vz_T
+ (\dot\gamma_t + \tfrac{g_t}{\gamma_t}) \hat \vu_t \Big) dt
+ \sqrt{2g_t}\, d\bar \vw_t,    
\end{equation}
which is exactly the ECSI reverse SDE drift form, see Sec. D in the supplement in~\cite{zhang2024exploring}. Note that, the above conduct establish that 
\emph{conditional SI, ECSI, and DDBM describe the same endpoint-conditioned marginals}. 
The difference lies only in training targets:  
ECSI/DDBM typically predict the posterior mean $\hat \vz_0$, 
while FM-based models predict the velocity $v$.  

\section{Additional Experimental Results}
In this section, we show additional experimental results that obmitted from the main paper due to limited spacing. 

\noindent
\textbf{Natural Image translation Visual Results.} In Fig.~\ref{fig:main_results_horsezebra}, we present qualitative results for the horse$\rightarrow$zebra and apple$\rightarrow$orange tasks. Overall, our method produces the most structurally consistent translations compared with existing approaches on these two classical I2I benchmarks. 

\noindent
\textbf{Additional Quantitative Results.}
In Fig.~\ref{fig:table2_appendix}, we provide additional quantitative evaluations (e.g., SSIM, CLIP-I, and CLIP-T) that were obmitted from the main paper. These figures show the full trade-off curves corresponding to Fig.~6 of the main text. As before, we plot text alignment (CLIP-T, x-axis) against a variety of fidelity and perceptual metrics (y-axis), including LPIPS, DINO, PSNR, SSIM, and CLIP-I. The expanded plots reveal the complete range of guidance settings used for each method. Consistent with the main-paper results, our method achieves a highly favorable balance between text adherence and structural preservation, performing competitively with—and often surpassing—state-of-the-art approaches such as FlowEdit.

We additionally compare the efficiency of all baselines across different tasks in Table~\ref{tab:mri-run}, Table~\ref{tab:nature-run}, and Table~\ref{tab:text-guide-run}. Rather than relying solely on Number of Function Evaluations (NFE), we measure the actual inference time on a single NVIDIA GeForce H200 GPU with a batch size of $1$. The results demonstrate that our method achieves inference speeds overall comparable to strong prior baselines. For text-guided I2I tasks (Table~\ref{tab:text-guide-run}), where NFE varies significantly due to hyperparameter settings (e.g., control strength), we omit a single NFE value and instead report the average running time across different settings to reflect the actual computational cost more accurately.

\noindent
\textbf{Additional Qualitative Results.}
In Fig.~\ref{fig:appendix_ct_id}, Fig.~\ref{fig:appendix_ct_wat_ood}, and Fig.~\ref{fig:appendix_ct_fat_ood}, we show additional visual results for in-domain and OOD MRI$\rightarrow$CT translation. Fig.~\ref{fig:appendix_t2I_scene_final} presents additional scene-level editing results using SSB, while Fig.~\ref{fig:appendix_t2I_object} visualize additional object level editing results.

\section{Limitations and Future Works}

\noindent\textbf{Medical Images.}
While SSB demonstrates strong 2D fully self-supervised MRI$\rightarrow$CT translation, a current limitation is that our medical pipeline operates on \emph{slice-wise} inputs. This choice is primarily motivated by the substantial computational cost of training full 3D vision encoders and 3D bridge models at clinical resolution. As with all 2D diffusion and translation frameworks, this may introduce mild inter-slice inconsistencies due to the absence of explicit through-plane anatomical modeling. Nevertheless, all baselines are evaluated under the same 2D setting, ensuring fair comparison. Extending SSB to 3D—either via slice-regularization techniques or fully 3D encoders and bridge architectures—remains an important direction for future work.

\noindent\textbf{General Domains Image Translation and Editing.}
Our method inherits two fundamental limitations stemming from its use of a strong external structure prior and from the nature of cross-domain representation alignment. First, the reliance on a geometry-preserving prior means that the model is highly effective for appearance-level edits—such as color, style, or seasonal changes—where the underlying object shape and scene layout should remain stable. However, this same inductive bias makes the method less suitable for large semantic transformations that require substantial changes in global geometry or object morphology. As shown in Fig.~\ref{fig:appendix_limitations} (\emph{Top}), when the target prompt demands drastic reinterpretation (\emph{e.g.}, transforming a small lizard into a large dragon), the model exhibits an inherent trade-off: stronger edits relax the structure prior but may also distort the background, alter object pose, or disrupt the scene layout, while preserving the structure suppresses the desired transformation.

Second, our approach struggles when the source and target domains have extremely large representation gaps. Domains such as silhouettes, segmentation masks, sketches, or other abstract inputs lack the rich texture and depth characteristics present in natural images. Since the structure prior was trained on natural imagery, it establishes meaningful correspondences only when both domains share compatible geometric information. When applied to highly stylized or symbolic inputs with no clear semantic grounding, the model may fail to generate realistic outputs or may collapse into painterly or stylized renderings instead of photorealistic results as shown in Fig.~\ref{fig:appendix_limitations} (\emph{bottom}), silhouettes to several RGB image style translations. 

These limitations highlight directions for future work, including developing adaptive structure priors that modulate geometric constraints based on target semantics, or designing cross-domain regularizers and multi-modal pretraining strategies that bridge the gap between abstract representations and natural images, enabling more robust abstract-to-photo translation.

\newpage
\begin{figure*}[t!] 
    \centering
    \includegraphics[width=1.\linewidth]{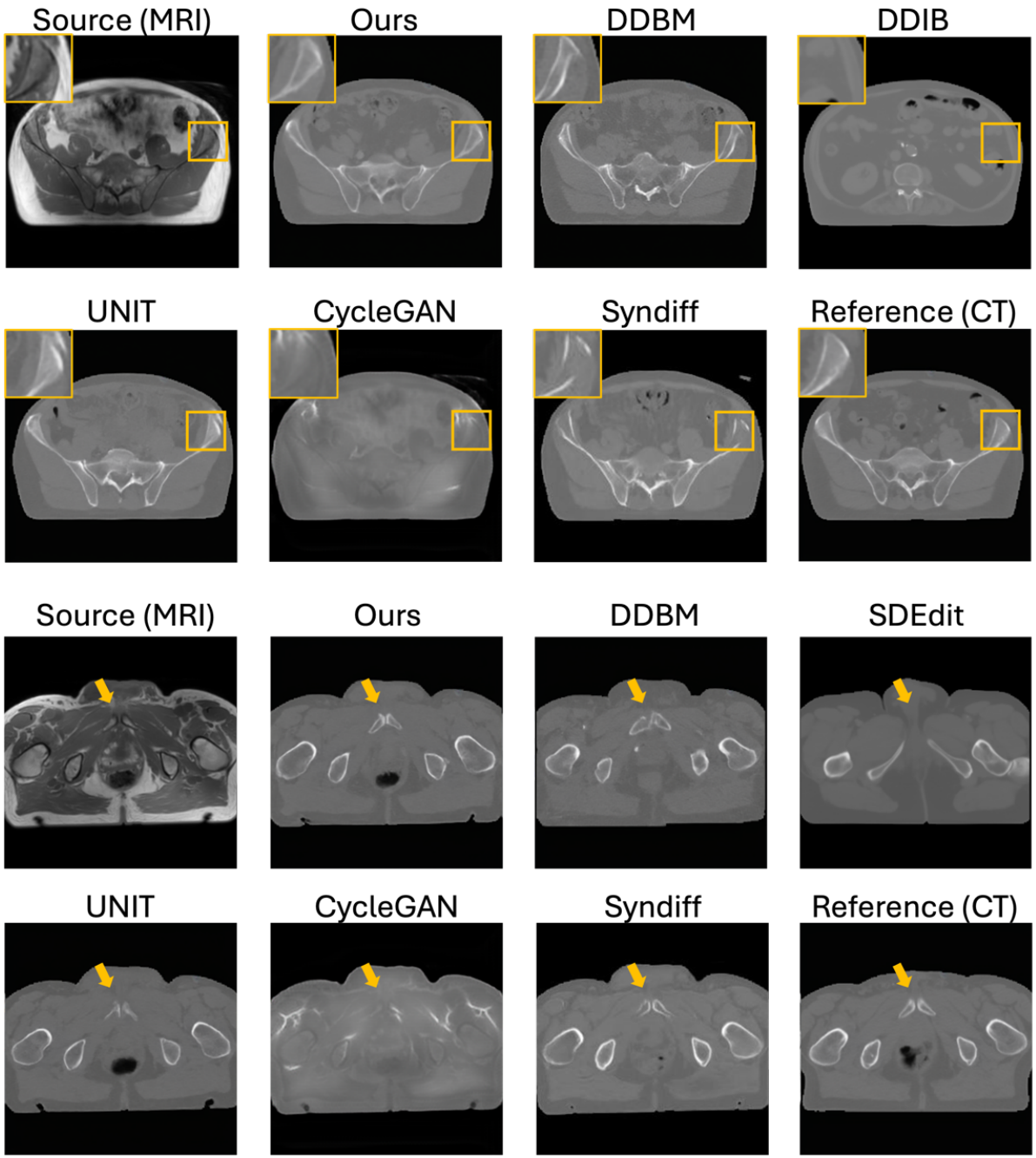}
    \vspace{-0.5em}
    \caption{\textbf{Additional Visual Results.} Additional qualitative results for in-domain MRI $\rightarrow$ CT translation. }
\vspace{-1em}    
    \label{fig:appendix_ct_id}
\end{figure*}

\newpage
\begin{figure*}[t!] 
    \centering
    \includegraphics[width=1.\linewidth]{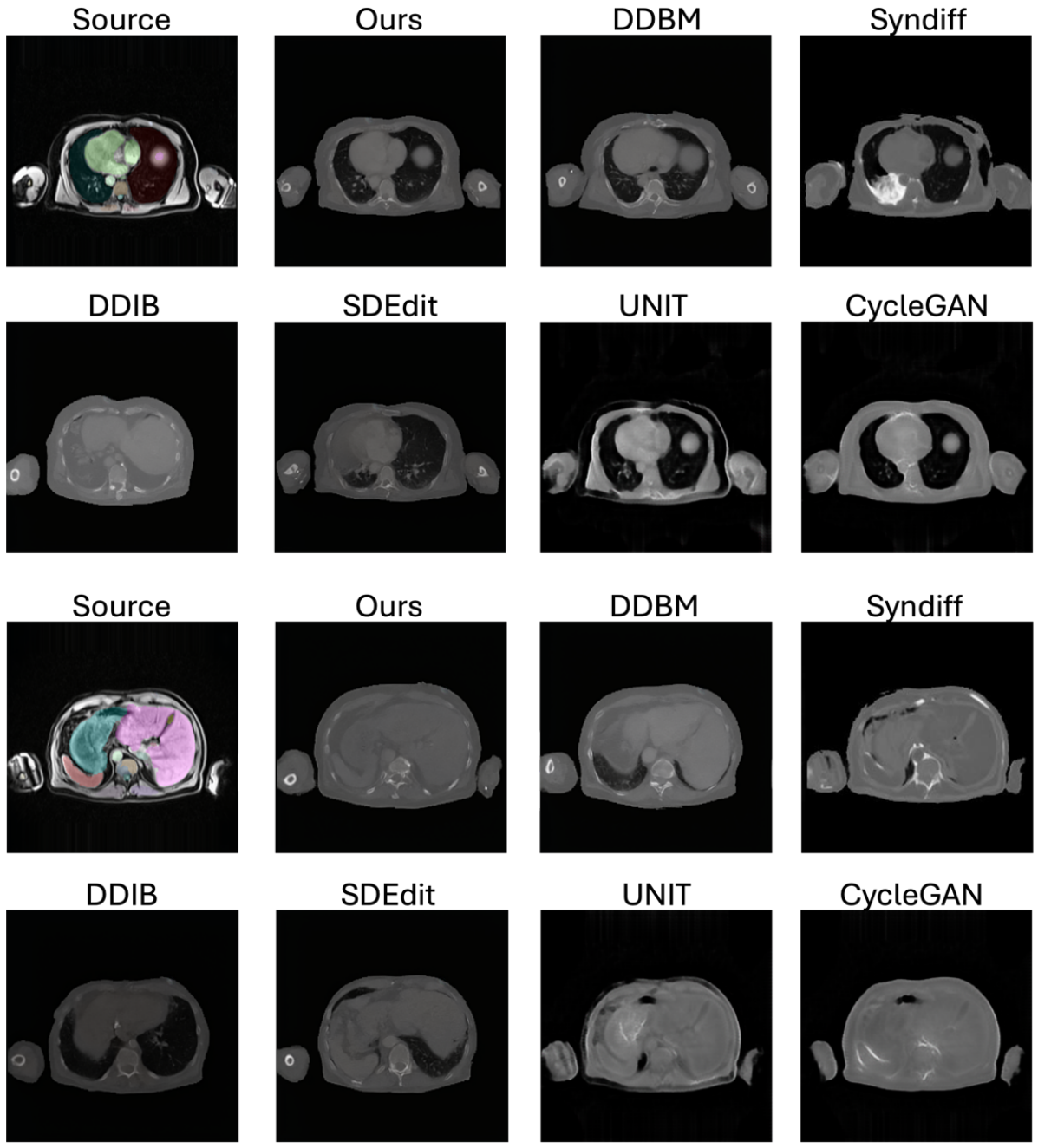}
    \vspace{-0.5em}
    \caption{\textbf{Additional Visual Results.} Additional qualitative results for out-of-distribution (OOD) MRI (UKBB whole-body water-series) $\rightarrow$ CT translation. Segmentation masks are overlaid on the MRI source images in OOD settings only to provide visual structural reference, since paired CT ground truth is unavailable. These masks are \emph{not} used during training or inference; they serve solely to illustrate anatomical fidelity without any segmentation supervision.}
\vspace{-1em}    
    \label{fig:appendix_ct_wat_ood}
\end{figure*}

\newpage
\begin{figure*}[t!] 
    \centering
    \includegraphics[width=1.\linewidth]{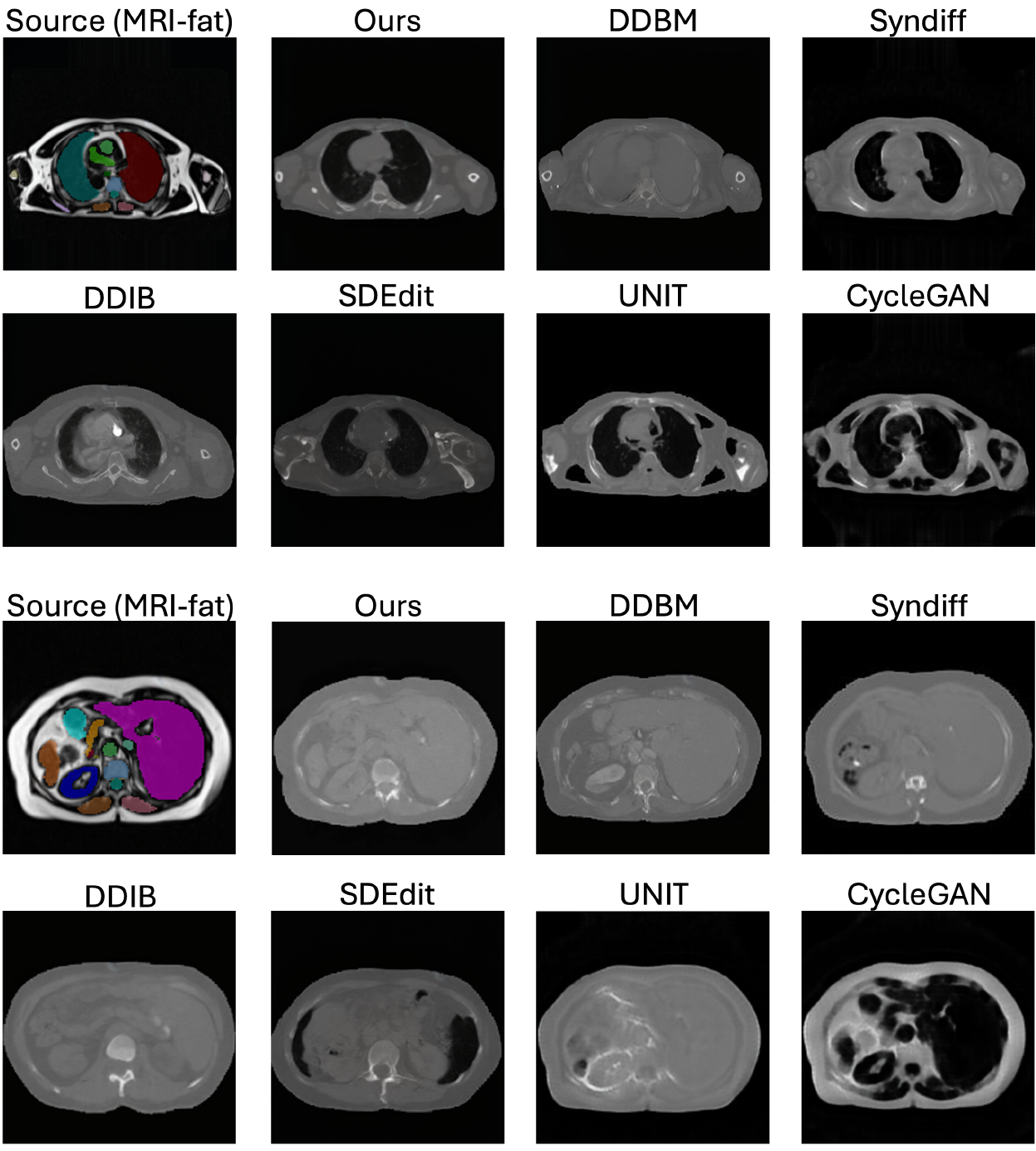}
    \vspace{-0.5em}
    \caption{\textbf{Additional Visual Results.} Additional qualitative results for out-of-distribution (OOD) MRI (UKBB whole-body fat-series) $\rightarrow$ CT translation. Segmentation masks are overlaid on the MRI source images in OOD settings only to provide visual structural reference, since paired CT ground truth is unavailable. These masks are \emph{not} used during training or inference; they serve solely to illustrate anatomical fidelity without any segmentation supervision.}
\vspace{-1em}    
    \label{fig:appendix_ct_fat_ood}
\end{figure*}

\newpage
\begin{figure*}[t!] 
    \centering
    \includegraphics[width=1.\linewidth]{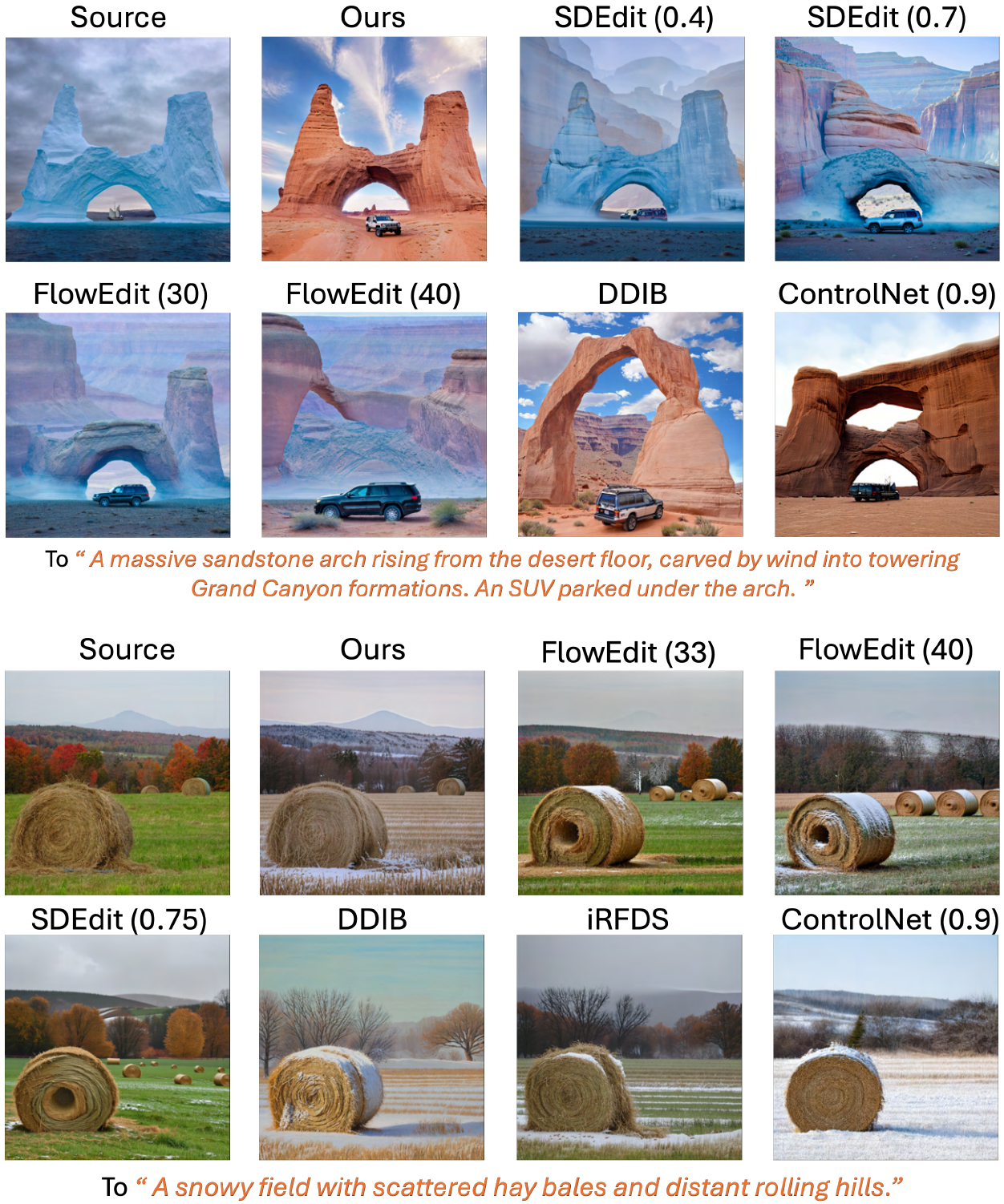}
    \vspace{-2em}
    \caption{\textbf{Additional Visualization Results.}
We present additional qualitative results for text-guided image scene editing using SD3-M~\cite{esser2024scaling}. Baseline methods are shown with numerical values in parentheses $(*)$, indicating the control strengths defined in their original formulations. For ControlNet, we use Canny edges as the conditioning signal.
For example, \emph{FlowEdit (33)} denotes a control parameter of $n_{\max}=33$, while \emph{ControlNet (0.7)} refers to a conditioning scale of $\alpha=0.7$, where larger values correspond to stronger Canny-based structural guidance.}
\vspace{-1em}    
    \label{fig:appendix_t2I_scene_final}
\end{figure*}

\newpage
\begin{figure*}[t!] 
    \centering
    \includegraphics[width=1.\linewidth]{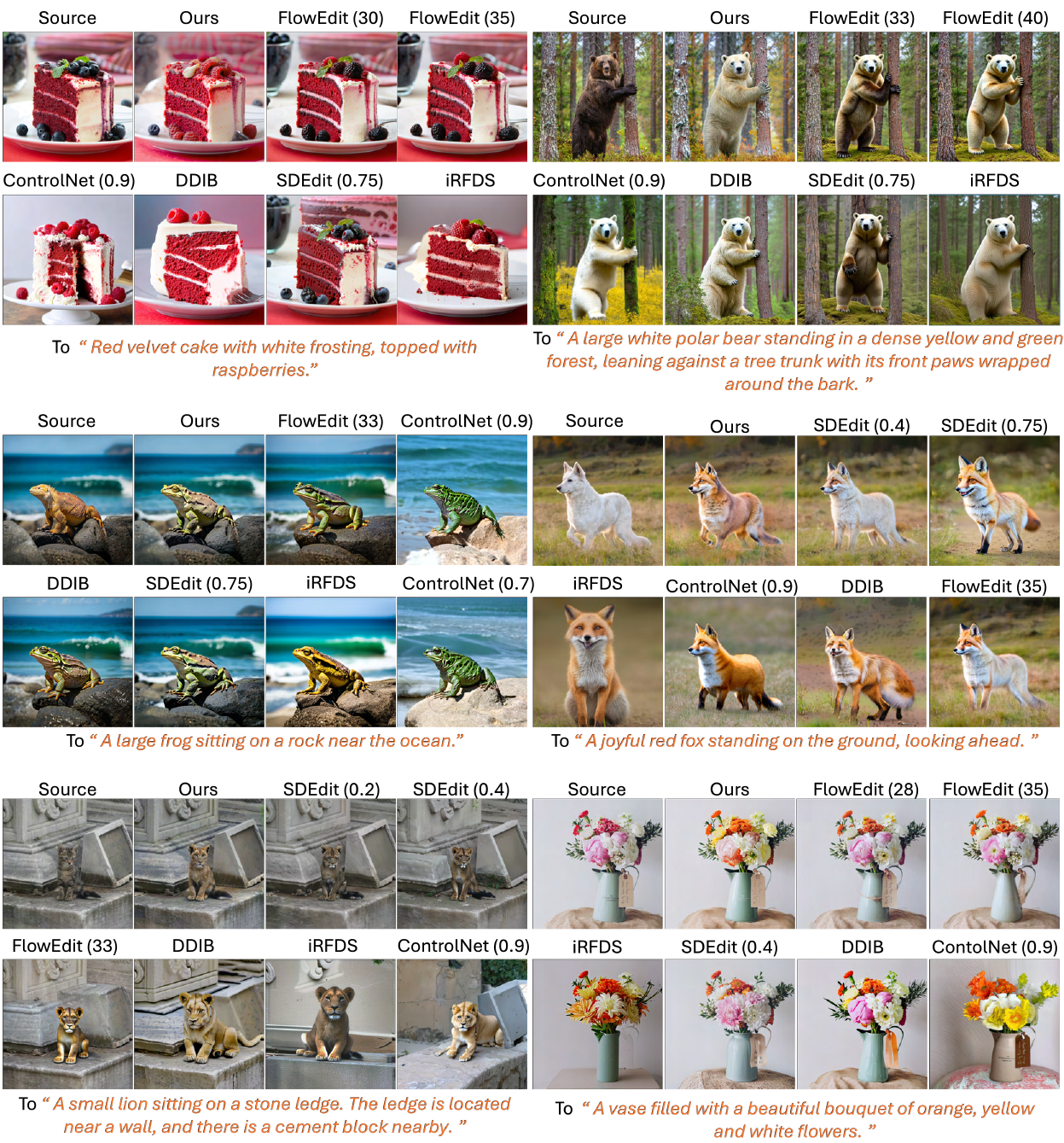}
    \vspace{-0.5em}
    \caption{\textbf{Additional Visualization Results.}
We present additional qualitative results for text-guided object editing using SD3-M~\cite{esser2024scaling}. Baseline methods are shown with numerical values in parentheses $(*)$, indicating the control strengths defined in their original formulations. For ControlNet, we use Canny edges as the conditioning signal.
For example, \emph{FlowEdit (33)} denotes a control parameter of $n_{\max}=33$, while \emph{ControlNet (0.7)} refers to a conditioning scale of $\alpha=0.7$, where larger values correspond to stronger Canny-based structural guidance.  }
\vspace{-1em}    
    \label{fig:appendix_t2I_object}
\end{figure*}

\newpage
\begin{figure*}[t!] 
    \centering
    \includegraphics[width=1.\linewidth]{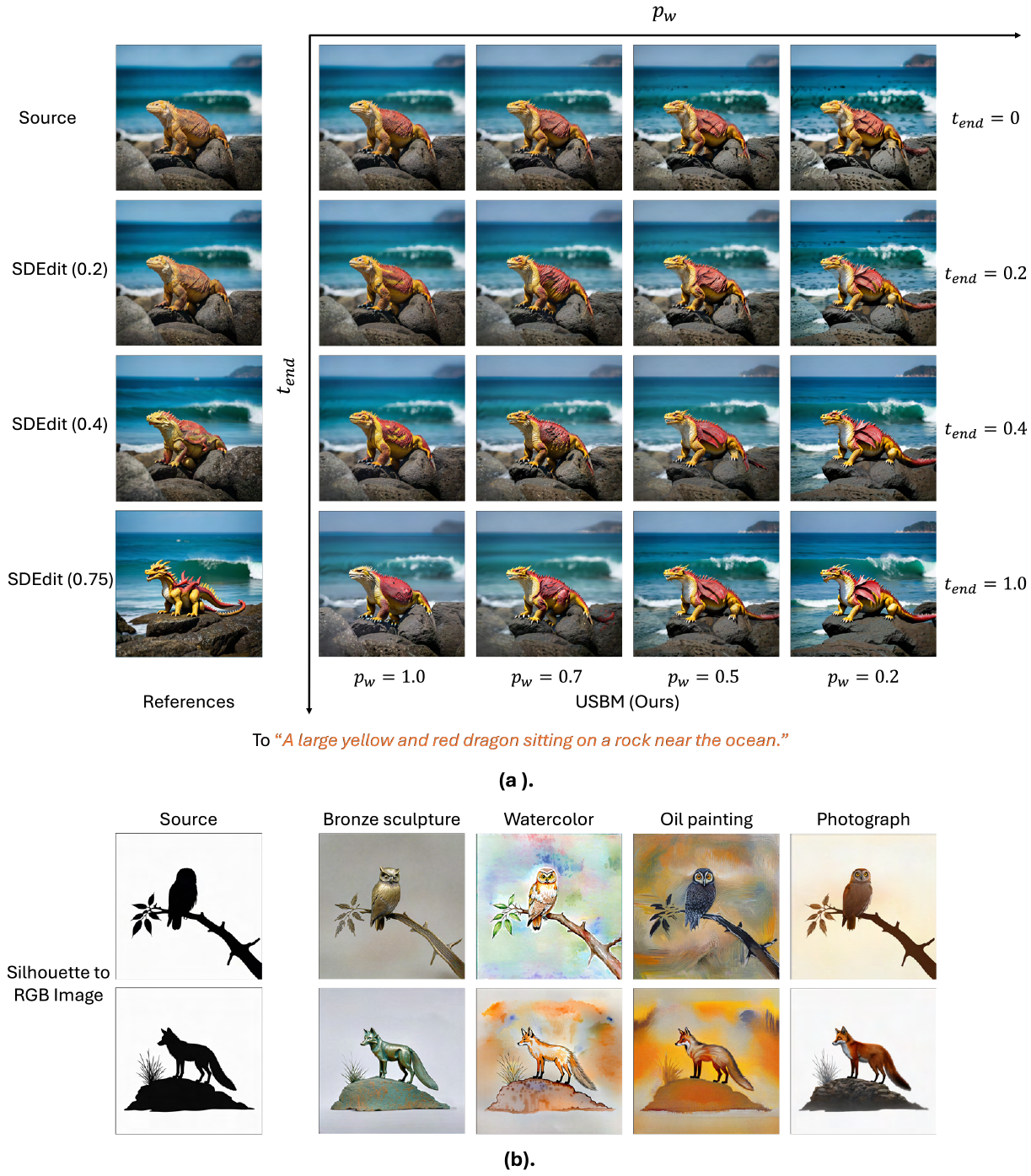}
    \vspace{-2em}
    \caption{\textbf{Limitations in Natural Image Editing and Translation.}\textbf{~\emph{Top}:} For large semantic changes (small lizard → large dragon), SSB trades off between preserving the source structure and fully following the target prompt, and strong edits can distort background geometry.\textbf{~\emph{Bottom}:} With highly abstract guidance such as silhouettes, the model tends to produce stylized, painterly animals rather than fully photorealistic translations.}
\vspace{-1em}    
    \label{fig:appendix_limitations}
\end{figure*}

\end{document}